\documentclass{article}

 \usepackage[preprint]{neurips_2026}

\usepackage[utf8]{inputenc} % allow utf-8 input
\usepackage[T1]{fontenc}    % use 8-bit T1 fonts
\usepackage{hyperref}       % hyperlinks
\usepackage{url}            % simple URL typesetting
\usepackage{booktabs}       % professional-quality tables
\usepackage{multirow}       % multi-row cells
\usepackage{graphicx}       % rotatebox / resizebox
\usepackage{amsfonts}       % blackboard math symbols
\usepackage{nicefrac}       % compact symbols for 1/2, etc.
\usepackage{microtype}      % microtypography
\usepackage[table]{xcolor}  % colors + \rowcolor / \cellcolor

\usepackage{changepage}

\definecolor{azure}{HTML}{74BEFF}

\usepackage{custom,customboxes}

\title{The Rate-Distortion-Polysemanticity Tradeoff in SAEs}

\author{%
      Tommaso Mencattini$^*$ \\
    EPFL\\
    \texttt{tommaso.mencattini@epfl.ch} \\
    \And
      Francesco Montagna\thanks{Equal contribution. Order based on seniority.} \\
    Institute of Science and Technology Austria\\
    \texttt{francesco.montagna@ist.ac.at} \\
  \And
  Francesco Locatello \\
Institute of Science and Technology Austria\\
}

\begin{document}

\maketitle

\begin{abstract}
    \looseness-1Sparse Autoencoders (SAEs) that can accurately reconstruct their input (minimizing distortion) by making efficient use of few features (minimizing the rate) often fail to learn monosemantic representations (highly interpretable), limiting their usefulness for mechanistic interpretability. In this paper, we characterise this tension in learning faithful, efficient, and interpretable explanations, introducing the Rate-Distortion-Polysemanticity tradeoff in SAEs. Under toy-modeling assumptions, we theoretically and empirically show that restricting the SAE to be monosemantic necessarily comes with an increase in rate and distortion. Assuming a generative model behind the input observations, we further demonstrate that the degree of polysemanticity of optimal SAEs is determined by the training data distribution, especially by the probability of features to co-occur. Finally, we extend the analysis to real-world settings by deriving necessary conditions that a polysemanticity measure should satisfy when the data-generating process is unknown, and we benchmark existing proxy metrics on SAEs trained on Large Language Models. 
    Taken together, our findings show that polysemanticity is a data problem that should be accounted for when addressing it at the architectural and optimization level. 
\end{abstract}

\begin{figure}[!h]
    \centering
    \includegraphics[width=\linewidth]{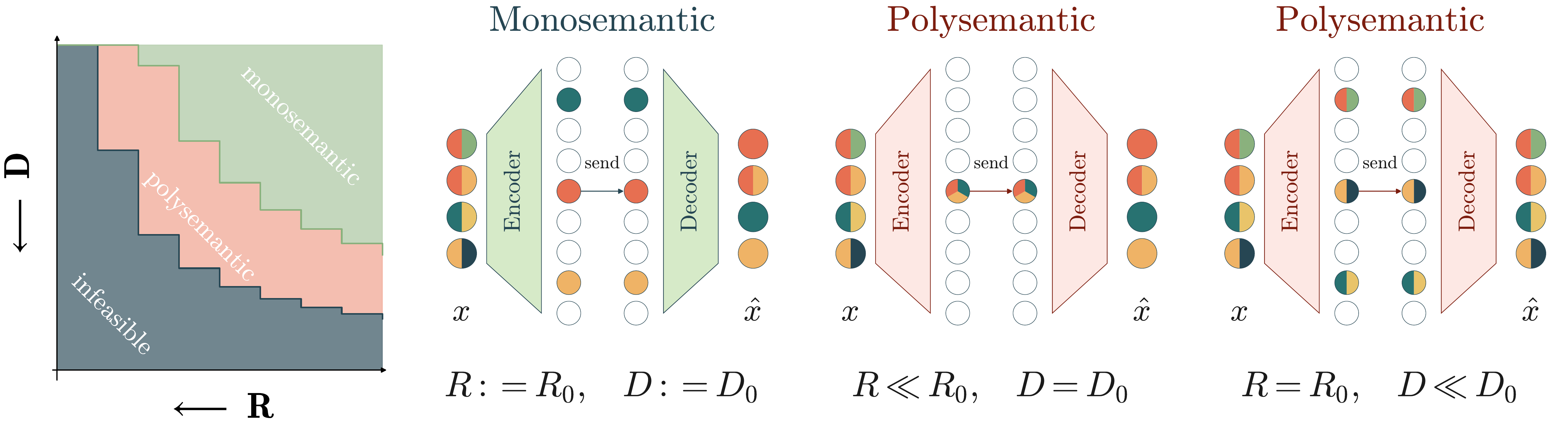}
    \caption{The rate-distortion-polysemanticity tradeoff. Enforcing monosemanticity on a SAE lifts the achievable rate-distortion frontier.
    }
    \label{fig:intro}
\end{figure}

% \FM{For @Tom: I use azure for figure related comments}
% \textcolor{azure}{\textbf{FM}: Comment on \cref{fig:intro}, Very nice. Some suggestions where I see it could be improved:
%     \begin{enumerate}
%         \item Concerning the color coding: I would make each input node of a different color (each color, correspond to a feature. And we carry those colors all over the figure. For the polysemantic node, you have a node of multiple colors. For the decoder part, the missed features should be white (we simply do not activate the feature we should)
%         \item In general, font sizes are too small in the encoder-decoder plots
%         \item Do we need the arrow with "transmit activation"? I mean, before and after the arrow the activations are identical, so is redundant. Moreover, I think it would be more sensible to have a. after the encoder the activated latent features b. after the decoder the reconstruction with those features 
%         \item The difference between the two polysemantic plots is in rate and distortion: taht part especially suffers from having tiny fonts. It is a very important bit of information, and so it should be highlighted: increase the font, and potentially use something (font color? not sure, but I leave the idea) to highlight R << R0, or D << D0
%     \end{enumerate}
%     }

\section{Introduction}

\looseness-1Sparse autoencoders have become a standard tool for explaining neural activations in terms of a dictionary of learned features \citep{cunningham2023saes, bricken2023monosemanticity}. A useful SAE explanation should satisfy three desiderata: it should be \emph{faithful}, so that reconstruction of the input is accurate; \emph{short}, so that few learned features suffice for a good reconstruction; and \emph{interpretable}, so that each of these features corresponds to a single concept, rather than to many unrelated ones. Current SAE evaluation captures only the first two desiderata, via sparsity and mean-squared error \citep{gao2024scaling}. The third, interpretability, is left implicit: as a result, SAEs that are optimal in sparsity and reconstruction are routinely polysemantic, learning features that mix multiple underlying concepts rather than cleanly aligning with one \citep{chanin2024absorption,chanin2025hedging,chanin2025sparsewrong}. 
% This is caused by monosemantic SAEs often not been able to reach optimality measured with sparsity and reconstruction \citep{chanin2025sparsewrong}. 
% \FM{Remove?}In an attempt to fix this, the community has proposed numerous architectural modifications and training protocols, including TopK, Gated, JumpReLU, BatchTopK, and Matryoshka SAEs \citep{gao2024scaling,rajamanoharan2024gated,rajamanoharan2024jump,bussmann2024batchtopksparseautoencoders,bussmann2025learningmultilevelfeaturesmatryoshka}. None of these solution had solved the underlying tension, suggesting that the problem is more fundamental in nature.

\looseness-1 The clash between faithfulness, shortness, and monosemanticity becomes familiar under an information-theoretic lens. Viewing an SAE as a communicator of explanations \citep{ayonrinde2024mdl} means we identify sparsity with \emph{rate} and reconstruction error with \emph{distortion} \citep{shannon1959,alemi2018brokenelbo} and SAEs' training is a rate-distortion problem. 
Unfortunately, reasoning in terms of rate and distortion is known to be insufficient to analyse additional qualitative properties of \textit{how} good reconstruction is achieved. \citet{blau2019rdp} made the point in lossy image compression, where low rate and low distortion admit \textit{perceptually} poor outputs, and similar arguments exist for representation learning~\citep{tschannen2018autoencoders}. For SAEs, we desire that good reconstruction is achieved with an efficient use of \textit{interpretable} features, rather than with their sacrifice; this motivates the introduction of \textit{polysemanticity} as an explicit additional property of the SAE, next to rate and distortion.

To formalize the interplay between rate, distortion, and polysemanticity in SAEs we posit a data-generating process (DGP) with ground-truth concept directions, a standard approach in theoretical SAE works \citep{chanin2024absorption,chanin2025hedging,ayonrinde2024mdl,chanin2025sparsewrong,chanin2026synthsaebench,rajendran2024from}; under this generative hypothesis, we can formalize a measure of polysemanticity, comparing learned features and groundtruth concepts.
% Empirically, prior work has identified mechanisms through which polysemantic features improve the sparsity-reconstruction tradeoff \citep{chanin2025hedging,chanin2024absorption,ayonrinde2024mdl}, but without a general measure of polysemanticity, these findings remain largely observational. Instead, we turn this into a fundamental property of SAEs by positing a data-generating process (DGP) with ground-truth concept directions, following a standard methodological strategy in theoretical SAE works \citep{chanin2024absorption,chanin2025hedging,ayonrinde2024mdl,chanin2025sparsewrong,chanin2026synthsaebench,rajendran2024from}. 
Once this structure is in place, a clear pattern emerges (\cref{fig:intro}): polysemanticity is a \textit{predictable} property of the optimal SAE, driven by the probability of co-occurring concepts in the DGP. When concepts frequently co-occur, a width- and rate-constrained SAE reduces distortion by representing them jointly. This is the \textit{rate-distortion-polysemanticity (RDP) tradeoff}. Our theoretical conclusions has real-world consequences: had we access to the DGP, a reliable polysemanticity measure would exhibit an RDP tradeoff. Then, even when the DGP is unknown, our findings give a theoretically grounded criterion for what a valid polysemanticity metric must look like, and let us choose between real-world interpretability metrics when they disagree.
% These toy-model insights do more than describe the DGP setting: they give us a theoretically grounded criterion for what a valid polysemanticity metric must look like, and let us choose between real-world interpretability metrics when they disagree.
\looseness-1We apply this criterion to open-source SAEs \citep{karvonen2025saebench}: no widely-used metric clearly passes, but the criterion still lets us rank candidates and rule out the ones that cannot be tracking polysemanticity at all. Surprisingly, the most expensive metrics, such as the LLM-based AutoInterp \citep{autointerp}, perform worse than far simpler baselines.  Overall, this reframes \textit{the emergence of polysemanticity as the interaction between architectures, optimization, and the likelihood of co-occurring events in the data.}
% \textcolor{teal}{Overall, this reframes where future work should focus: \textit{reducing polysemanticity requires architectures and optimization strategies that can mitigate the risks that co-occurrences in the training distribution impose on interpretability.}}
% \textcolor{teal}{Overall, this reframes \textit{the emergence of polysemanticity as the interaction between architectures, optimization, and data. Future strategies can mitigate the risks that co-occurrences in the training distribution impose on interpretability.}}
% is less a matter of designing new architectures than of curating the training data and the metrics we use to evaluate them}.
To sum up, this paper tackles the following problem:

% Defining the RDP frontier, \FL{therefore?}however, requires quantifying polysemanticity. Prior work has identified empirical mechanisms through which polysemantic features improve the sparsity--reconstruction tradeoff \citep{chanin2025hedging,chanin2024absorption,ayonrinde2024mdl}, but without a general measure of polysemanticity, these findings remain largely observational. \FL{Thus, to the best of our knowledge, this trade-off was not explicitly characterized before}. The difficulty is that measuring polysemanticity requires a semantic reference (e.g., the underlying concepts), whereas in real-world SAEs the relevant concepts are typically implicit and unknown. To overcome this limitation, we posit a data-generating process (DGP) with ground-truth concept directions, following a standard methodological strategy in theoretical SAE work \citep{chanin2024absorption,chanin2025hedging,ayonrinde2024mdl,chanin2025sparsewrong,chanin2026synthsaebench, rajendran2024from}. Once this structure is in place, a clear pattern emerges (\cref{fig:intro}): in the toy model, polysemanticity is a \textit{predictable} property of the optimal SAE, driven by the probability of co-occurring concepts in the data-generating process. When concepts frequently co-occur, a width- and rate-constrained SAE reduces distortion by representing them jointly. These toy-model insights also constrain the properties that a faithful proxy for polysemanticity must satisfy in real-world SAEs. To sum up, this paper tackles the following problem:
\vspace{-2mm}
\begin{emptyroundbox}
\paragraph{Problem definition.}
We study the RDP tradeoff by quantifying the cost that monosemanticity imposes on rate and distortion, and show that polysemanticity is a predictable optimal response to the co-occurrence pattern of concepts in the data-generating process.
\end{emptyroundbox}
\vspace{-1.5mm}

\textbf{Contributions. } (i) We formalize the rate-distortion-polysemanticity tradeoff and prove that tightening the polysemanticity budget worsens the achievable rate or distortion. (ii) We construct a tractable toy model in which the probability of co-occurrence of features in the data-generating process predictably determines optimal polysemanticity. (iii) We derive necessary monotonicity conditions for any real-world polysemanticity proxy and use them to evaluate proxies on Gemma and GPT-2 SAEs. The practical takeaway: reducing polysemanticity requires architectures and optimization strategies aware of the risk that co-occurrences in the training distribution imposes on interpretability.
% attention to the training distribution, not only to architectures or optimization heuristics.

\section{Background}
\looseness=-1In this section, we define main concepts to develop our theory. First, we introduce sparse autoencoders and their information theory interpretation. Second, we introduce the DGP and define \textit{polysemanticity} and the \textit{Rate-Distortion-Polysemanticity} function. Throughout, we denote matrices with uppercase letters $X$, vectors with lowercase, bold letters $\mathbf x$, and scalar with lowercase letters $x$. Extensive discussion of related works in information theory, machine learning, and SAEs can be found in Appendix~\ref{relworks}.

\subsection{SAEs as compressed explanations: from reconstruction to rate-distortion}\label{sec:sae_comm}

Sparse autoencoders were designed to map neural network activations from a small, polysemantic basis to a large, sparse, and approximately monosemantic one \citep{bricken2023monosemanticity, cunningham2023saes}. In practice, a SAE maps an activation $\mathbf x\in\mathbb{R}^d$ to a sparse latent code $\mathbf z \in\mathbb{R}^m$ and reconstructs $\mathbf x$ from that code. With encoder and decoder weights $W_{\mathrm{enc}},W_{\mathrm{dec}}\in\mathbb{R}^{m\times d}$, biases $\mathbf b_{\mathrm{enc}},\mathbf b_{\mathrm{dec}}$, a sparsifying nonlinearity $\sigma$ (e.g.\ ReLU or TopK), and an optional sparsity regularizer $\mathcal S$ (commonly the $\ell_1$ norm), the SAE is trained to solve
\begin{equation}\label{eq:sae}
\operatorname*{arg\,min}_{\theta\in\Theta} \; \mathbb E\!\left[\|\mathbf x-\hat{\mathbf x}_\theta\|_2^2+\lambda\,\mathcal S(\mathbf z_{\theta}(\mathbf x))\right]
\quad\text{s.t.}\quad
\begin{cases}
\mathbf z_\theta(\mathbf x)=\sigma(W_{\mathrm{enc}}\mathbf x+\mathbf b_{\mathrm{enc}}), \\
\hat{\mathbf x}_\theta=W_{\mathrm{dec}}^\top \mathbf z_\theta(\mathbf x)+\mathbf b_{\mathrm{dec}},
\end{cases}
\end{equation}
where $\theta = (W_{\mathrm{enc}}, W_{\mathrm{dec}}, \mathbf b_{\mathrm{enc}}, \mathbf b_{\mathrm{dec}}) \in \Theta$ collects all SAE parameters. Henceforth, the decoder is referred as \textit{dictionary} and its columns as \textit{atoms}.

The task of learning good parameters $\theta$ is usually described as a sparse coding or reconstruction problem \citep{gao2024scaling,hindupur2025projectingassumptionsdualitysparse,oneill2025computeoptimalinferenceprovable}. We adopt instead a communication view,.
% in the spirit of \citet{ayonrinde2024mdl}. 
Consider a sender and a receiver who have agreed in advance on the SAE parameters $\theta$. To transmit an activation $\mathbf x$, the sender computes the sparse code $\mathbf z_\theta(\mathbf x)$ and transmits only the code; the receiver applies the decoder to recover $\hat{\mathbf x}_\theta \approx \mathbf x$. The \emph{rate} is the cost of transmitting the code and the \emph{distortion} is the reconstruction error; the parameters $\theta$ are shared infrastructure and do not count toward the rate. 
% This differs from \citet{ayonrinde2024mdl}, who transmit the decoder as part of the message via two-part MDL coding \citep{grunwald2007mdl}, 
This perspective places our setup in Shannon's classical rate-distortion framework \citep{shannon1959},
\begin{equation}\label{eq:classical_rdf}
R(D_0)=\inf_{p(\hat{\mathbf x}\mid \mathbf x):\,\mathbb E[d(\mathbf x,\hat{\mathbf x})]\le D_0} I(\mathbf x;\hat{\mathbf x}),
\end{equation}
which gives the smallest rate needed to stay below a distortion budget $D_0$.  For SAEs, the natural analogues of communication cost and fidelity are the number of active latents and the mean squared reconstruction error, yielding the SAE rate-distortion frontier
\begin{equation}\label{eq:rdf}
R^\star(D_0):=\inf_{\theta\in\Theta}\{R(\theta):D(\theta)\le D_0\}
\quad\text{s.t.}\quad
\begin{cases}
R(\theta) := \mathbb E[\|\mathbf z_\theta(\mathbf x)\|_0], \\
D(\theta) := \mathbb E[\|\mathbf x-\hat{\mathbf x}_\theta\|_2^2] ,
\end{cases}
\end{equation}
where $R(\theta)$ measures how short (efficient) the explanation is, and $D(\theta)$ how faithful it is. In the next section, we argue why this novel two-axis picture is still insufficient, and we define a measure of polysemanticity to solve this problem.

\subsection{A semantic reference frame and the RDP objective}\label{sec:dgp_poly_rdp}

\begin{figure}[t]
    \centering
    \includegraphics[width=\linewidth]{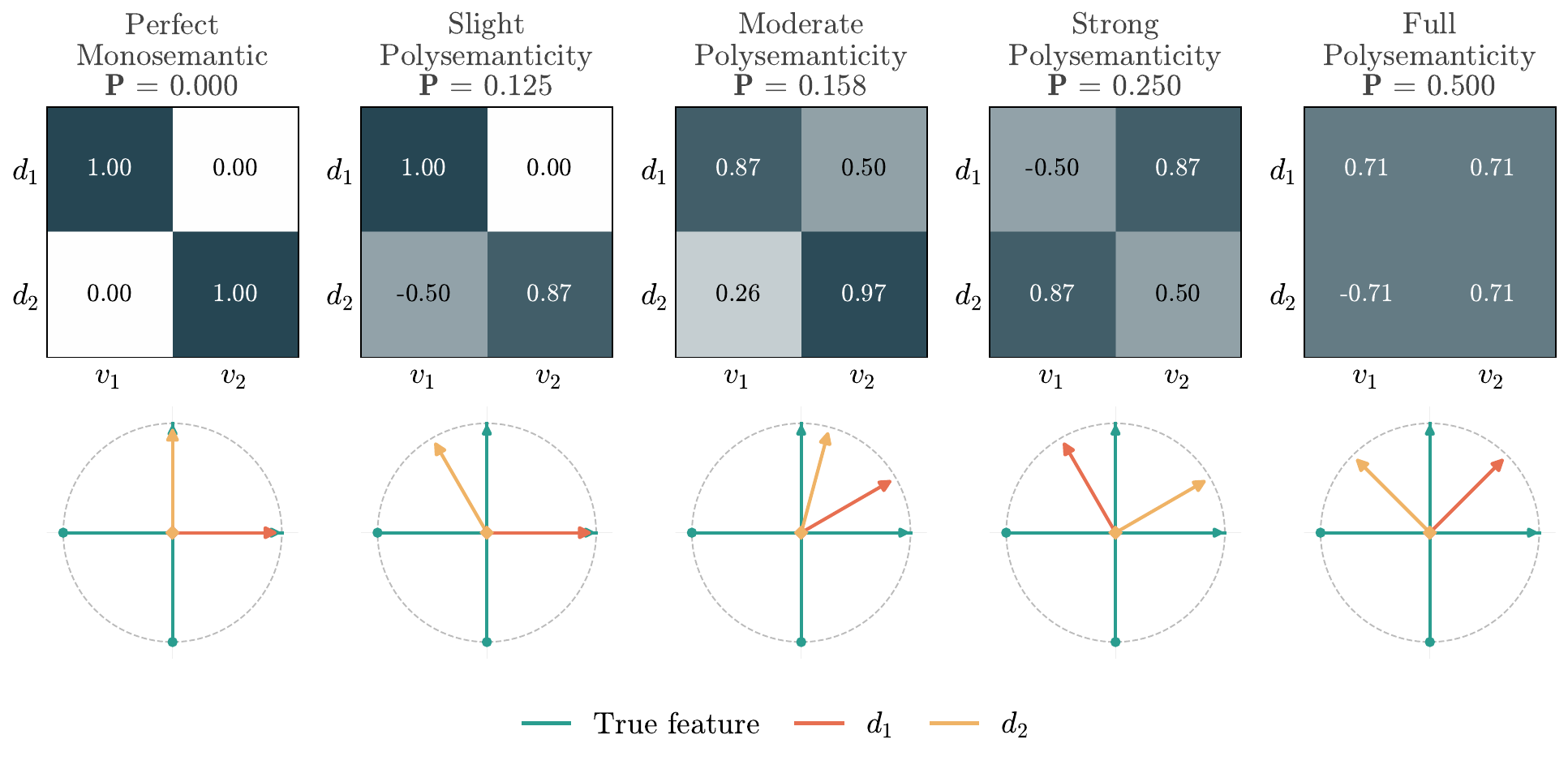}
    \caption{Five representative values of the polysemanticity metric $P$ (\cref{def:polysemanticity}). The first line illustrates the cosine similarity matrix between ground truth ($v_i$) and dictionary's ($d_i$) features.}
    \label{fig:poly_examples}
\end{figure}

% \textcolor{azure}{\textbf{FM}:Comment on \cref{fig:poly_examples}:
% \begin{enumerate}
%     \item Font size increase
%     \item If I don't mistake, in the main text you use $w$, not $d$. If thats the case, be consistent, both in the matrices and in the legend of the geometric representation of polysemanticity
%     \item Maybe, remove the tick on the vertical and horizontal axis 
% \end{enumerate}}

The rate-distortion view is still incomplete for interpretability: an SAE can attain low rate and low distortion while learning heavily mixed features \citep{chanin2024absorption,chanin2025hedging}, and conversely monosemantic SAEs achieve suboptimal rate-distortion values \citep{chanin2025sparsewrong}. This mirrors the failure mode identified by \citet{tschannen2018autoencoders,blau2019rdp}, where the rate-distortion curve is opaque to qualitatively important properties (e.g., representation usefulness or perceptual quality) that must be explicitly modeled as an additional axis. Our main claim is that SAEs have the same structural problem, with polysemanticity as the missing axis. Making this coordinate measurable, however, requires a semantic reference: polysemanticity is not a function of the SAE alone, but of the alignment between its atoms and the underlying concepts. To formalize this notion, we therefore introduce a data-generating process (DGP).

\paragraph{Data Generating Process. } Let $V=(\mathbf v_1,\dots,\mathbf v_n)\in\mathbb{R}^{d\times n}$ be ground-truth concept directions, and let $\mathbf c=(c_1,\dots,c_n)\in\{0,1\}^n\sim p_{\mathbf c}$ indicate which concepts are active. We assume
\begin{equation}\label{eq:dgp}
\mathbf x=\sum_{\ell=1}^n c_\ell \mathbf v_\ell=\sum_{\ell\in S}\mathbf v_\ell,
\qquad
S=\{\ell:c_\ell=1\},
\end{equation}
where $S$ collects the indices of the concepts active in $\x$ generation, and $\mathcal D = (V,p_{\mathbf c})$ collects all parameters of the DGP. This DGP is a closed-form instantiation of the Linear Representation Hypothesis \citep{elhage2022superposition,park2024lrh}, under which an LLM represents concepts as (nearly) orthogonal directions in activation space and only a small subset fires on any given token. Now that we have fixed the semantic reference (DGP), we can define polysemanticity.

\begin{definition}[Polysemanticity]\label{def:polysemanticity}
Let $V \in \mathbb{R}^{d \times n}$ collect $n$ ground-truth concept directions. Consider $m$~SAE's feature rows $\hat{\mathbf v}_1,\dots,\hat{\mathbf v}_m$. Let $C \in \mathbb{R}^{m\times n}$ be the matrix of cosine similarities $c_{i\ell} = \cos(\hat{\mathbf v}_i, \mathbf v_\ell)$. The polysemanticity of the dictionary $(\hat{\mathbf v}_i)_{i=1}^m$ is:
\begin{equation}\label{eq:polysemanticity}
P(C) := \frac{1}{m}\sum_{i=1}^m\Bigl(1-\max_\ell c_{i\ell}^2\Bigr) \; \in [0,1].
\end{equation}
\end{definition}

% Applying $P$ to the encoder and decoder, we denote their cosine-similarity matrices by $C_{\mathrm{enc}}$ (with entries $c^{\mathrm{enc}}_{i\ell}=\langle \mathbf w^{\mathrm{enc}}_i,\mathbf v_\ell\rangle/(\|\mathbf w^{\mathrm{enc}}_i\|\,\|\mathbf v_\ell\|)$) and $C_{\mathrm{dec}}$ (defined analogously from the decoder rows), and track both sides jointly via
% \begin{equation}\label{eq:polysemanticity_joint}
% P_{\mathrm{joint}}(\theta):=P(C_{\mathrm{enc}})+P(C_{\mathrm{dec}}).
% \end{equation}
Intuitively, $P(C)$ measures the average ``spread'' of each feature row across the ground-truth concepts: it vanishes exactly when every nonzero row is aligned, up to sign, with a single concept, and grows as rows become mixtures of several concepts. \cref{fig:poly_examples} illustrates five representative values of $P$ on a $2\times 2$ dictionary. We can now state the paper's central object.

\begin{definition}[Rate-distortion-polysemanticity function]\label{def:rdpf}
Fix the data generating process $\mathcal D$ and let $\Theta$ denote a class of SAEs with rate, distortion, and polysemanticity functionals $R(\theta)$, $D(\theta)$, and $P(\theta)$. For a distortion budget $D_0\ge 0$ and a polysemanticity budget $P_0\ge 0$, the \emph{rate-distortion-polysemanticity function} is the smallest rate achievable while simultaneously meeting both budgets,
\begin{equation}\label{eq:rdpf}
R^\star(D_0,P_0):=\inf_{\theta\in\Theta}\bigl\{R(\theta):D(\theta)\le D_0,\;P(\theta)\le P_0\bigr\}.
\end{equation}
\end{definition}

\begin{figure}[!t]
    \centering
    \includegraphics[width=\linewidth]{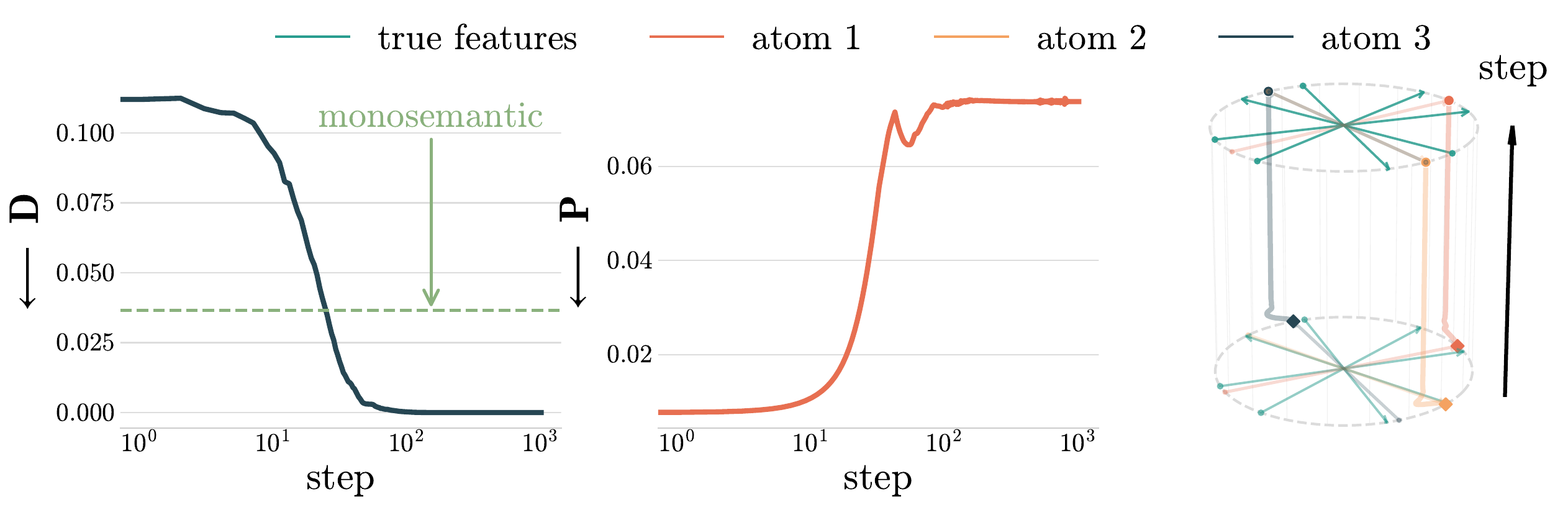}
    \caption{Training dynamics of a width-3 Top-2 SAE on a DGP with four orthogonal concepts $\mathbf v_i \in \R^2$ fired uniformly at probability $p=0.2$. During the optimization, the dictionary rotates from a nearly monosemantic initialization to a polysemantic configuration in order to minimize distortion.}
    \label{fig:ib_training}
\end{figure}

Intuitively, the RDP function is the best (smallest) achievable rate by an SAE that is required to meet distortion and polysemanticity constraints. 

Equipped with this third axis, we can finally ask what ordinary SAE evaluation cannot: \emph{how much extra rate or distortion must be paid to enforce semantic purity?} \cref{fig:ib_training} illustrates this with a small example: the toy SAE rotates away from its monosemantic initialization, with distortion dropping past anything a monosemantic dictionary can reach at the same rate. The gap between the two curves is the price monosemanticity would cost. In other words, the optimal SAE is polysemantic by design, not by accident. What determines how large this price is? \cref{sec:toy_model} shows it is the interplay between the rate budget and the concept co-occurrence pattern in the DGP.

\section{A Toy Model of Polysemanticity}\label{sec:toy_model}

In this section, we give an informal, narrative account of why this gap exists and which properties of the DGP control its size. A rigorous treatment is deferred to \cref{sec:toy}. The decisive quantity is the co-occurrence pattern of concepts: when concepts fire together often enough, the SAE can lower distortion by bundling them into polysemantic features that monosemantic dictionaries cannot represent under rate constraints. 
 % The full formal treatment including exact characterizations of the monosemantic frontier \FM{@Tom: what is the \textit{monosemantic frontier}? Consider replacing with more transparent wording} and quantitative rate and distortion-tax theorems \FM{@Tom: rate and distortion \textit{tax} doesn't yet have a meaning here}, is deferred to \cref{sec:toy}.
\cref{fig:synthetic_rdp} previews the three claims we will make: tightening any one of the three budgets (rate, distortion, polysemanticity) forces the other two up, and the $R$--$D$ frontier at fixed $P_0$ takes the staircase shape that the analysis predicts under aligned-atoms assumptions. The rest of the section walks through these claims: that the trade-off must exist (\cref{sec:toy_tradeoff}), and can be quantified in terms of additional cost (call it \textit{premium}) in rate or distortion imposed by monosemanticity (\cref{sec:toy_premium}), and that the premium is governed by the probability of co-occurrence of concepts in the DGP (\cref{sec:toy_cooccurrence}).

Throughout the section, we work under four simplifying assumptions, whose full statement and scope are discussed in detail in \cref{sec:setting}: intuitively, they reduce the class of SAEs we studied to a realistic yet analytically tractable family. The most consequential requirement is orthogonality of the ground-truth concept directions (\cref{ass:orthogonal_features}), $\langle \mathbf v_i,\mathbf v_\ell\rangle=\delta_{i\ell}$: under orthogonality, observing $\mathbf x$ is equivalent to observing the set $S=\{\ell:c_\ell=1\}$ of active concepts (see \cref{eq:dgp}), so the law $p_{\mathbf x}$ reduces to the co-occurrence of concepts described by the probability $p_S$, which is the object our analysis will track. We also set $\mathbf b_{\mathrm{enc}}=\mathbf b_{\mathrm{dec}}=0$ (\cref{ass:zero_bias}) and, for fully monosemantic SAEs, take $W_{\mathrm{enc}}=W_{\mathrm{dec}}^\top$ (\cref{ass:tiedness}); the latter is without loss of generality under \cref{def:polysemanticity} (proven in \cref{prop:tied_wlog} in the appendix). Finally, we impose that encoder and decoder rows are zero/one combinations of true concepts (\cref{ass:aligned_atoms}), so that each SAE reconstructing $\mathbf x$ with $\mathbf x_\theta$ can be thought of as a map $S\mapsto \hat S_\theta$, where
\begin{equation}\label{eq:aligned_reconstruction_toy}
\hat{\mathbf x}_\theta=\sum_{\ell\in\hat S_\theta}\mathbf v_\ell.
\end{equation}
This last assumption is a genuine simplification rather than without loss of generality; it makes the problem of studying the tradeoff combinatorial and hence tractable. These assumptions are simplifications, not claims about real activations: their role is to isolate the response of an optimal SAE to the co-occurrence pmf $p_S$ from geometric and optimization complications. The conclusions we draw are about the principles co-occurrence imposes on optimal codes, not about the exact closed forms holding in real SAEs.

We now develop our three main claims in order. The next subsection establishes the trade-off qualitatively, as a structural property of \cref{def:rdpf} that requires none of the above assumptions; the full formal statement and proof are in \cref{sec:toy}.

\subsection{The RDP trade-off cannot be optimized away}\label{sec:toy_tradeoff}

Our first claim is qualitative: the three budgets in \cref{def:rdpf} cannot be jointly tightened for free. This follows from the definition of $R^\star(D_0,P_0)$ as an infimum over a nested family of feasible sets, and does not rely on any of the assumptions stated above.

\begin{theorem}[Monotonicity of the RDP frontier, informal]\label{thm:monotonicity_informal}
Fix a DGP and an SAE class $\Theta$. Then $R^\star(D_0,P_0)$ and its distortion dual $D^\star(R_0,P_0)$ are nonincreasing in $P_0$. Intuitively: if we want better monosemanticity,
\begin{enumerate}[(i)]
    \item but we are not willing to pay the price via worse reconstruction (distortion budget $D_0$), then we have to raise the required rate, i.e., use features less sparsely;
    \item but we are not willing to pay the price using more features for reconstruction (rate budget $R_0$), then we have to raise the distortion.
\end{enumerate}
\end{theorem}
The formal statement is presented in  \cref{thm:monotonicity} with a rigorous proof.

% The first monotonicity is the familiar rate-distortion trade-off. The two involving $P_0$ are the new content: they say that the polysemanticity axis cannot be added for free on top of rate and distortion.

\begin{callout}
Tightening monosemanticity is never a free lunch. An SAE that matches a polysemantic baseline on reconstruction while being strictly more monosemantic must pay for it on the rate axis, and vice versa. This is what \cref{fig:synthetic_rdp} reports empirically: in each panel, moving the third budget down lifts the Pareto frontier in the remaining two.
\end{callout}

The monotonicity is qualitative and tells us nothing about how big the premium is. For a narrow SAE on a DGP with heavy co-occurrence, it can be large; for other DGPs, it can be zero. Quantifying it requires more structure.

Next, we turn the qualitative monotonicity statement into a quantitative lower bound. In the main text, we study the minimum rate increase to achieve better monosemanticity. The companion distortion tax has a similar characterization.
% Under orthogonality and aligned atoms, the monosemantic class obeys a conservation law $D+R=\mathbb{E}\|\mathbf c\|_0$; this yields a closed-form rate tax that any monosemantic SAE matching a polysemantic baseline must pay. The companion distortion tax and the proofs are in \cref{sec:toy}.

\subsection{Quantifying the premium: the rate tax}\label{sec:toy_premium}

To quantify the premium, we use the closed-form reconstruction identity available under \cref{ass:orthogonal_features} and \cref{ass:aligned_atoms}: given a DGP with concepts $\mathbf c$ activating with probability $p_{\mathbf c}$, for any monosemantic SAE $\theta$ that can reconstructs a set $I\subseteq\{1,\dots,n\}$ of groundtruth concepts, \cref{prop:monosemantic_distortion} gives
\[
D(\theta)+R(\theta)=\mathbb{E}_{p_{\mathbf c}}\|\mathbf c\|_0.
\]
Distortion and rate are coupled by a conservation law: on the monosemantic class, every unit of rate reduces distortion by exactly one unit, and every concept omitted from the represented set costs its marginal probability in distortion. This has a concrete consequence for the rate tax that monosemanticity imposes. (We show the result under \cref{ass:orthogonal_features} and \ref{ass:aligned_atoms}.)

\begin{theorem}[Rate tax, informal]\label{thm:rate_tax_informal}
Fix a class $\Theta$ of TopK-SAEs with rate budget $K=k$. Let $D_\infty$ be the best distortion achievable at rate $\le k$ by an arbitrarily polysemantic SAE $\theta^*$. Assume $\theta^*$ is actually polysemantic (mixes concepts). Then, if there is a monosemantic SAE $\theta$ with distortion $\leq D_{\infty}$, it has a rate larger than $k$. More precisely:
\[
R(\theta)\;\ge\;\mathbb{E}_{p_{\mathbf c}}\|\mathbf c\|_0-\delta\;>\;k,
\]
where $\delta$ is the upper bound of the distortion that can be achieved by a monosemantic SAE that is still better than $D_\infty$.
\end{theorem}

In \cref{sec:formal_theorems} we formalize the statement (\cref{thm:rate_tax}) and present a rigorous proof.

Let $D_{\infty}$ be the optimal distortion achieved without polysemanticity constraints yet under a rate budget. Vice versa, let $\delta\leq D_{\infty}$ be achieved requiring strict monosemanticity, but no rate constraint---$\delta$ is the worse degradation a monosemantic SAE can achieve while still being better than $D_{\infty}$. Then, the lower bound is a function of the DGP alone through \textit{(i)} the expected sparsity $\mathbb{E}_{p_{\mathbf c}}\|\mathbf c\|_0$ and \textit{(ii)} the attainable monosemantic distortion levels $\delta$. Concretely, the more mass $p_{\mathbf c}$ places on large co-occurring sets, the larger $\mathbb{E}_{p_{\mathbf c}}\|\mathbf c\|_0$, and the larger the rate tax imposed by monosemanticity.

\begin{callout}
Whenever the optimal SAE at a given rate is polysemantic, any monosemantic SAE that reconstructs at least as well must be strictly less efficient. The DGP determines by how much: the gap is tight to the co-occurrence structure encoded in $\mathbb{E}_{p_{\mathbf c}}\|\mathbf c\|_0$ and the monosemantic distortion levels $\delta$.
\end{callout}

The dual statement is that, at fixed rate, a monosemantic SAE incurs a strictly larger distortion. Both taxes share the same driver, the conservation law $D+R=\mathbb{E}\|\mathbf c\|_0$ on the monosemantic class, and together explain the staircase shape of the $R$ versus $D$ frontier at $P_0=0$ visible in the rightmost panel of \cref{fig:synthetic_rdp}: the feasible monosemantic $(R,D)$ pairs are a discrete set indexed by which concepts are omitted, so the frontier is a staircase rather than a line. The exact staircase is described by \cref{thm:exact_monosemantic_frontier}.

\begin{figure}[tbp]
    \centering
    \includegraphics[width=\linewidth]{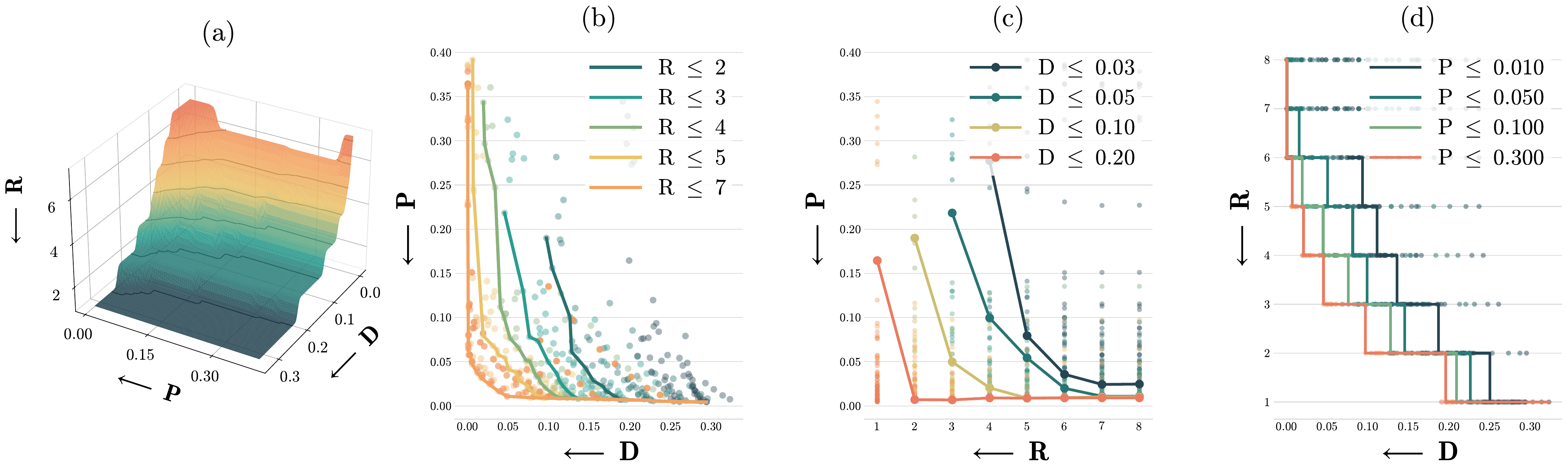}
    \caption{Empirical RDP frontier for TopK SAEs on a synthetic DGP ($n=20$ concepts in $\R^6$, $\sim 0.6$ features active per sample; sweep over rate $k\in\{1,\dots,8\}$, monosemanticity penalty $\lambda\in\{0,\dots,100\}$, $7$ seeds). \textbf{Left:} full $(P,D,R)$ surface. \textbf{Center-left.} $P$ vs $D$ Pareto front at rate bounds $k\le k_0$: with a tighter rate, monosemanticity ($P \downarrow$) is improved at the price of worse distortion. \textbf{Center-right.} $P$ vs $R$ frontiers at distortion budgets $D\le D_0$: monosemanticity ($P \downarrow$) is improved at the price of worse rate. \textbf{Right.} $R$ vs $D$ frontier at polysemanticity budgets $P\le P_0$: asking for better monosemanticity shifts the Pareto front to the right (worse rate \textit{and} distortion).}
    \label{fig:synthetic_rdp}
\end{figure}
% \textcolor{azure}{\textbf{FM}: Comments on \cref{fig:synthetic_rdp}.
% \begin{enumerate}
%     \item Font size both on the axis and on the legends. For the first plot, you can increase the font size and having a tick every 0.10, for instance (to avoid overcrowding the axes)
%     \item Need to make 4 subfigures, to make it easier to reference later. 
% \end{enumerate}
% }

The taxes tell us the premium is nonzero, but not why. The next subsection identifies the mechanism: writing the SAE loss in closed form as a function of $p_S$ (probability of co-occurring concepts) reveals optimal polysemanticity as a response to joint occurrence of concepts.

\subsection{Co-occurrence of concepts drives optimal polysemanticity}\label{sec:toy_cooccurrence}

The previous two subsections establish that a premium exists and bound it from below. They do not yet explain \emph{why} polysemanticity is the mechanism that a rate-constrained SAE uses to reduce distortion. The answer comes from writing the reconstruction loss in closed form as a function of the probability of co-occurrence of concepts in the DGP distribution.

\begin{lemma}[Closed-form SAE loss]\label{lem:loss} Let Assumptions \ref{ass:orthogonal_features}, \ref{ass:zero_bias}, \ref{ass:tiedness}, \ref{ass:aligned_atoms} satisfied. Then, the mean squared reconstruction error of a TopK SAE $\theta$ is:
\[
\mathcal L(\theta,K)=\mathbb{E}_{p_S}\bigl[\,|S|+|\hat S_\theta|-2\,|S\cap \hat S_\theta|\,\bigr].
\]
\end{lemma}

The quantity inside the expectation is the symmetric difference $|S\,\triangle\,\hat S_\theta|$. Loss is minimized by choosing, for each source event $S$ weighted by $p_S$, a reconstruction $\hat S_\theta$ that overlaps $S$ as much as possible under the width and rate constraints. Because, under a narrow SAE assumption $K$, the number of active features, is smaller than the number of concepts in the DGP, no deterministic rule can match $\hat S_\theta=S$ on every event; the SAE must commit to a small family of events it can reconstruct and pay the symmetric-difference cost on the events it does not match.

\begin{callout}
An observation $\mathbf x(S)$ is an observation of a co-occurrence event $S$. The optimal SAE reconstructs the co-occurrence events with the largest $p_S$, and pays symmetric-difference distortion on the rest. The DGP, through $p_S$, directly picks which features the SAE will learn.
\end{callout}

This is the mechanism behind the premium of the previous subsection. When the DGP concentrates mass on events in which a pair of concepts $\mathbf v_i,\mathbf v_j$ fires jointly, a single polysemantic atom writing $\mathbf v_i+\mathbf v_j$ matches every such event at unit rate cost; two monosemantic atoms for the same pair cost two units of rate, or one unit of distortion if one is dropped. The inequality between these options is an inequality between probabilities in $p_S$.

\paragraph{A three-concept example.} To make this precise, we illustrate the case $n=3$, $m=2$, $K\in\{1,2\}$ in closed form (worked out in detail in \cref{sec:coccurrence}). A monosemantic SAE $\theta_M^{jk}$ represents two of the three concepts; hedged codes $\theta_H^{ij,k}$ replace one monosemantic atom dedicated to $\mathbf v_j$ with a joint atom writing $\mathbf v_i+\mathbf v_j$; feature-splitting codes $\theta_S^{ij,i}$ combine a joint atom with a fallback atom for one of its components. Applying \cref{lem:loss} we can compute the close form loss of any monosemanitic, hedged, and feature-splitting SAE: this  gives explicit inequalities in the probability mass function entries of $p_S$, which characterize when a polysemantic SAE beats its monosemantic competitor. For $K=2$, for instance,
\begin{equation*}
    \begin{split}
        &\mathcal L(\theta_M^{jk})-\mathcal L(\theta_H^{ij,k})>0\iff p_{ij}+p_{ijk}>p_j+p_{jk},\\
        &\mathcal L(\theta_M^{jk}) - \mathcal L(\theta_{S}^{ij, j}) >0 \iff  p_{ij} > p_k + p_{ik} + p_{jk}
    \end{split}
\end{equation*}
and analogous inequalities hold for $K=1$; the full set is in \cref{sec:coccurrence}.

\begin{callout}
A polysemantic atom writing $\mathbf v_i+\mathbf v_j$, whether by hedging or by splitting, is optimal exactly when the co-occurrence probabilities $p_{ij}$ and $p_{ijk}$ dominate the marginals of the concepts it would otherwise displace. Polysemanticity is an optimal response to joint occurrence, not a pathology of the optimizer.
\end{callout}

Two consequences are worth emphasizing. First, the driver is joint \emph{occurrence}, not correlation: the inequalities involve raw probabilities $p_{ij},p_{ijk}$, not centered quantities. A counterexample in \cref{sec:coccurrence} exhibits hedging under zero correlation and no hedging under strong correlation, so previous correlation-based accounts \citep{chanin2025hedging} are at best a proxy. Second, both hedging and feature splitting emerge in narrow SAEs under the same mechanism, rather than splitting being a wide-SAE-only phenomenon: both reduce symmetric-difference distortion under a rate cap, and the inequalities that certify them are of the same type.

The three-concept case is tight enough to enumerate all competitor codes in closed form; larger $n$ rapidly becomes combinatorial, and the general characterization of which joint atoms enter the optimum is the content of \cref{thm:exact_monosemantic_frontier} together with its polysemantic generalization sketched in \cref{sec:coccurrence}. The empirical sweep of \cref{fig:synthetic_rdp} corroborates the qualitative predictions at $n=20$: the $R$ versus $D$ frontier at $P_0=0$ is a staircase, tightening rate raises the monosemanticity floor (center-left panel), and the premium grows with the mass $p_S$ places on multi-concept events.

\begin{emptyroundbox}
\paragraph{Takeaway.}
Under the assumptions listed above, the RDP trade-off is a structural property of optimal SAEs, not an optimization artifact. Its size at a given rate is a function of the co-occurrence pmf $p_S$. Reducing polysemanticity without raising rate or distortion therefore requires changing the training distribution, not the architecture or the optimizer.
\end{emptyroundbox}

\section{Interpretability in The Wild: a Meta-Analysis of SAE's Metrics}\label{sec:real_world}
Our theory so far rests on a data-generating process, which is unrealistic in practical SAE use cases: the ground-truth concepts $V$ are unobservable, and $P(\theta)$ of \cref{def:polysemanticity} cannot be computed. As interpretability is the main reason practitioners train SAEs, the community has responded with \emph{proxies}, scalar functionals of an SAE intended to track interpretability without ground-truth concepts \citep{karvonen2025saebench,minegishi2025pseval}. These proxies routinely disagree, and there is no principled way to choose among them. We argue that the RDP theory supplies one: any genuine proxy for $P$ must respect the same tradeoff, a \emph{meta-metric consistency criterion}. % usable without ever accessing the true $\mathcal D$.

The criterion follows directly from \cref{thm:monotonicity}. If a candidate proxy $\hat P:\Theta\to\mathbb{R}_{\ge 0}$ were faithful to $P$, \cref{thm:monotonicity} would force the minimum of $\hat P$ under a joint rate-distortion budget to decrease as either budget is relaxed. We instantiate this criterion in the two following tests, which only require access to a finite sweep of $N$ SAEs, $\{(R_i,D_i,\hat P_i)\}_{i=1}^N$, rather than to the true DGP $\mathcal D$.
The first test is \textit{local}: it inspects every pair of SAEs in the sweep in which one has a tighter budget than the other, and checks whether the proxy orders them in the direction \cref{thm:monotonicity} predicts.
\begin{definition}[RDP monotonicity-violation rate]\label{def:violation_rate}
Call a pair $(i,j)$ \emph{dominated} when $R_i\le R_j$, $D_i\le D_j$, and $(R_i,D_i)\neq(R_j,D_j)$, namely when SAE $i$ sits under a tighter joint budget than $j$. Then $V(\hat P)$ is the fraction of dominated pairs on which the proxy is inverted: $V(\hat P) := |\mathcal D|^{-1}\sum_{(i,j)\in\mathcal D}\mathbf 1\{\hat P_i<\hat P_j\} \in [0,1].$
\end{definition}

A faithful proxy has $V=0$; $V=\tfrac12$ is the random baseline; $V>\tfrac12$ means the proxy systematically moves \emph{opposite} to the RDP-predicted direction. The second test is \textit{global}: rather than inspecting individual pairs, it measures whether proxy values decrease monotonically as rate and distortion budgets relax across the whole sweep.

\begin{definition}[RDP rank correlation]\label{def:rho}
The sign-flipped Spearman correlation between the joint budget rank and the proxy value: $\rho(\hat P) := -\mathrm{corr}_{\mathrm{Sp}}\bigl(\mathrm{rank}(R_i)+\mathrm{rank}(D_i),\,\hat P_i\bigr) \in [-1,1].$
\end{definition}

\begin{figure}[!htbp]
    \centering
    \includegraphics[width=\linewidth]{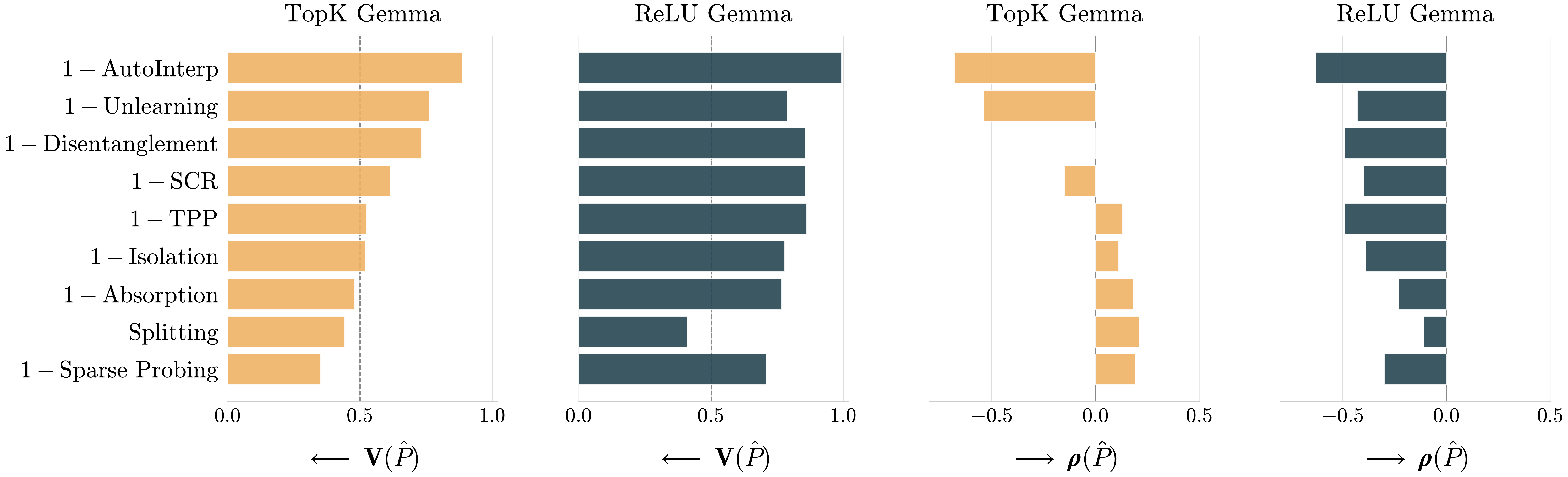}
    \caption{$V$ (left, lower is better; dashed line = random baseline $V=\tfrac12$) and $\rho$ (right, higher is better) for the nine SAEBench proxies on TopK and ReLU Gemma.}
    \label{fig:proxy_audit_bars}
\end{figure}

A faithful proxy has $\rho$ near $+1$; $\rho\approx 0$ signals no trend; $\rho<0$ signals global inversion. We report both because they diagnose different failure modes: $V$ catches pair-level reversals that a monotone trend can mask, $\rho$ catches global flatness that a few correctly-ordered pairs can mask. 

We apply the audit to TopK and ReLU SAEs on Gemma activations, evaluated under the SAEBench metrics of \citet{karvonen2025saebench}. The SAEs and their SAEBench scores are obtained directly from \citet{karvonen2025saebench} rather than retrained by us. \cref{fig:proxy_audit_bars} summarizes the results, which separate the metrics into three regimes on TopK Gemma. A first group is \emph{anti-correlated} with the RDP direction: AutoInterp ($V=0.89$, $\rho=-0.68$) and Unlearning ($V=0.76$, $\rho=-0.54$) respond to budget tightening in the wrong direction, and on ReLU Gemma the pattern extends to every SAEBench metric except Splitting. Strikingly, AutoInterp, by far the most expensive metric in the suite, relying on an LLM-as-judge, is the most severely misaligned. A second group, Isolation, TPP, and Absorption, is effectively \emph{uncorrelated}, hovering near the random baseline. Only Sparse Probing ($V=0.35$, $\rho=+0.19$) and Splitting ($V=0.44$, $\rho=+0.21$) show \emph{weak agreement} with the RDP direction, and neither strongly. No widely-adopted SAE metric clearly tracks the RDP-predicted direction, suggesting that measuring polysemanticity in real-world SAEs may require metrics built on semantically controlled data rather than on downstream probes or LLM-based judgments. Full numbers and per-proxy scatter plots are provided in the appendix (\cref{fig:gemma_rd_by_p}, \cref{fig:relu_gemma_rd_by_p}, \cref{tab:proxy_audit}).

\begin{emptyroundbox}
\paragraph{Takeaway.}
In practice, proxy metrics for polysemanticity do not faithfully track the RDP trade-off, making progress difficult to quantitatively assess. Surprisingly, expressive metrics like AutoInterp are most misaligned, suggesting they may be unreliable measures of polysemanticity. We speculate that semantically controlling the training data may be a viable way forward to quantify improvements in interpretability.  
\end{emptyroundbox}

\section{Conclusion}
% We show that without explicit accounting for monosemanticity of the features, SAE trained to sparsely reconstruct features are by necessity polysemantic. 
\looseness-1Under toy model assumptions, we demonstrate that polysemanticity is a response of the SAE to exploit co-occurrence of concepts in the training distribution in order to minimize the reconstruction error under sparsity constraints. Given an appropriate measure of polysemanticity, we qualitatively and quantitatively characterize this Rate-Distortion-Polysemanticity tradeoff.  Finally, we argue that in real-world settings, when polysemanticity is tracked without ground-truth concepts, a good proxy must exhibit the tradeoff. This gives a necessary condition that we use to evaluate widely adopted measures of polysemanticity, which we show to be mostly insufficient in this respect. Our takeaway is that, for the goals of mechanistic interpretability, reducing polysemanticity requires synergy between architectures, optimization strategies, and data, the latter being underexplored until now. %aware of the risk that co-occurrences of concepts in the training distribution impose.

\section{Acknowledgments}
FM is funded by the Chan Zuckerberg Initiative. TM is financed from the project “Building Energy Systems on causal reasoning (BOSS)“, funded within the “Technologies and Innovations for the Climate-Neutral City” (TIKS) Programme of the Austrian Research Promotion Agency (FFG).

\bibliographystyle{plainnat}
\bibliography{biblio}

\newpage
\appendix
\section{Notation}\label{app:notation}

\begin{table}[h]
\centering
\small
\begin{tabular}{@{}ll@{}}
\toprule
\textbf{Symbol} & \textbf{Meaning} \\
\midrule
\multicolumn{2}{@{}l}{\textit{Dimensions}} \\
$d$ & dimension of the ambient activation space \\
$m$ & SAE width (number of latent coordinates) \\
$n$ & number of ground-truth concepts in the DGP \\
$K$ & TopK sparsity budget, when applicable \\
\midrule
\multicolumn{2}{@{}l}{\textit{Sparse autoencoder}} \\
$\mathbf x \in \mathbb{R}^d$ & input activation \\
$\hat{\mathbf x} \in \mathbb{R}^d$ & SAE reconstruction of $\mathbf x$ \\
$\mathbf a(\mathbf x) \in \mathbb{R}^m$ & preactivation \\
$\mathbf z(\mathbf x) \in \mathbb{R}^m$ & latent code (message) \\
$W_{\mathrm{enc}}, W_{\mathrm{dec}} \in \mathbb{R}^{m\times d}$ & encoder / decoder weight matrices \\
$\mathbf b_{\mathrm{enc}} \in \mathbb{R}^m,\ \mathbf b_{\mathrm{dec}} \in \mathbb{R}^d$ & encoder / decoder biases \\
$\sigma$ & nonlinear sparsifying function, e.g.\ ReLU or TopK \\
$\theta = (W_{\mathrm{enc}}, W_{\mathrm{dec}}, \mathbf b_{\mathrm{enc}}, \mathbf b_{\mathrm{dec}})$ & parameter tuple of an SAE \\
$\Theta$ & class of SAEs under consideration \\
\midrule
\multicolumn{2}{@{}l}{\textit{Data generating process}} \\
$V = (\mathbf v_1,\ldots,\mathbf v_n) \in \mathbb{R}^{d\times n}$ & matrix of ground-truth concept directions \\
$\mathbf v_\ell \in \mathbb{R}^d$ & $\ell$-th ground-truth concept direction \\
$\mathbf c \in \{0,1\}^n$ & binary indicator vector of active concepts \\
$c_\ell \in \{0,1\}$ & indicator that concept $\ell$ is active \\
$p_{\mathbf c}$ & joint pmf of the concept vector $\mathbf c$ \\
$S = \{\ell : c_\ell = 1\}$ & index set of active concepts for a sample \\
$p_S$ & pmf of co-occurring concept events \\
$\hat S_\theta \subseteq \{1,\ldots,n\}$ & set of concepts reconstructed by the SAE \\
\midrule
\multicolumn{2}{@{}l}{\textit{Rate, distortion, polysemanticity}} \\
$I(\mathbf x;\hat{\mathbf x})$ & mutual information between source and reconstruction \\
$d(\cdot,\cdot)$ & distortion measure in the classical Shannon setup \\
$R(\theta) = \mathbb{E}[\|\mathbf z_\theta(\mathbf x)\|_0]$ & rate of an SAE (expected $\ell_0$ of the code) \\
$D(\theta) = \mathbb{E}[\|\mathbf x - \hat{\mathbf x}_\theta\|_2^2]$ & distortion of an SAE (MSE) \\
$R^\star(D_0)$ & rate-distortion function of the class $\Theta$ \\
$C_{\mathrm{enc}} = \cos(W_{\mathrm{enc}} V^\top)$ & encoder / ground-truth cosine-similarity table \\
$C_{\mathrm{dec}} = \cos(W_{\mathrm{dec}} V^\top)$ & decoder / ground-truth cosine-similarity table \\
$P(C)$ & row-wise max-alignment polysemanticity of a cosine table \\
$P_{\mathrm{joint}}(\theta) = P(C_{\mathrm{enc}}) + P(C_{\mathrm{dec}})$ & joint encoder-decoder polysemanticity \\
$R^\star(D_0, P_0)$ & rate-distortion-polysemanticity (RDP) function \\
$D^\star(R_0, P_0)$ & distortion dual of the RDP function \\
\bottomrule
\end{tabular}
\caption{Notation used throughout the paper.}
\label{tab:notation}
\end{table}

%%%%%%%%%%%%%%%%%%%%%%%%%%%%%%%%%%%%%
%%%%%%%%%%%%%%%%%%%%%%%%%%%%%%%%%%%%%

\section{Limitations}\label{app:limitations}
A key limitation of our analysis is that the strongest results are obtained under a stylized toy model: orthogonal ground-truth concepts, binary/aligned atoms, zero biases, and other tractability assumptions that isolate co-occurrence effects from geometry and optimization. Accordingly, our theorems identify principles governing optimal codes under a known DGP, not exact characterizations of real SAEs trained on neural activations. In realistic settings, the ground-truth concepts and co-occurrence law are unobserved, so polysemanticity cannot be measured directly and must be inferred through proxies. This means our metric analysis delivers necessary consistency conditions for such proxies, but not a sufficient validation that any given proxy truly captures polysemanticity in the wild. More broadly, our empirical audit should be read as evidence of misalignment in current evaluations, rather than as a definitive measurement of interpretability itself. Future work should relax the toy assumptions, test the theory under richer generative structures, and develop benchmarks with semantically controlled data that permit stronger empirical validation.

%%%%%%%%%%%%%%%%%%%%%%%%%%%%%%%%%%%%%
%%%%%%%%%%%%%%%%%%%%%%%%%%%%%%%%%%%%%

\section{Related Works}\label{relworks}

\paragraph{Rate-Distortion Theory.} Our work is grounded in rate-distortion theory, originally developed in information theory and later revisited in machine learning. Classical rate-distortion theory defines the minimum price (rate) one must pay to represent a signal within a prescribed distortion budget \citep{Cover, Stone, shannon1959}. In machine learning, this framework has been studied extensively through autoencoders, most prominently variational autoencoders, where the \emph{rate} measures how much information the latent $z$ preserves about the input $x$, and the \emph{distortion} measures reconstruction error \citep{alemi2018brokenelbo, burgerVAE}. However, \citet{tschannen2018autoencoders,blau2019rdp} show that this two-axis view is incomplete whenever an additional qualitative property, such as perceptual quality in image compression or representation usefulness in downstream tasks. The same limitation of the rate-distortion theory reappears in mechanistic interpretability: \citet{ayonrinde2024mdl} and \citet{chanin2025sparsewrong} argue that the sparsity-reconstruction trade-off, used to describe SAEs \citep{gao2024scaling, rajamanoharan2024jump}, is insufficient, since it fails to capture interpretability and would judge monosemantic SAEs as inferior, yet what exactly is missing remains unspecified. Our work formalizes polysemanticity as that missing axis and develops a precise rate-distortion-polysemanticity theory for SAEs.

\paragraph{Polysemanticity, Superposition and LRH in neural networks.} The second core theme of our paper is \textit{polysemanticity} \citep{olah2020}. A standard explanation for why neurons turn polysemantic, responding to multiple distinct concepts, is \emph{superposition}: when many sparse features must be packed into a limited-dimensional activation space, they share directions, so mixed neurons emerge from capacity pressure rather than optimization accident \citep{elhage2022superposition,scherlis2022capacity}, with monosemantic neurons reserved for the most important concepts. Both notions should be distinguished from the \textit{Linear Representation Hypothesis} (LRH), which posits that concepts are encoded as approximately linear directions in activation space, without itself implying monosemanticity or superposition \citep{park2024lrh}. Whereas prior work treats polysemanticity in LMs as a qualitative phenomenon, we turn it into a measurable quantity by defining an exact metric for its instantiation in SAEs in light of the data generating process.

\paragraph{Structural Analysis of Sparse Autoencoders.}
Sparse autoencoders were introduced precisely in response to the polysemantic frame described above: if concepts are represented as linear directions, but individual neurons are polysemantic, then an overcomplete learned dictionary may recover a more monosemantic feature basis \citep{bricken2023monosemanticity,cunningham2023saes}. A growing body of work, however, has shown that SAEs do not reliably achieve this ideal. \citet{chanin2024absorption} and \citet{ayonrinde2024mdl} show that SAEs can improve sparsity by turning otherwise monosemantic features into structurally degenerate ones, through mechanisms such as feature absorption and splitting, and \citet{chanin2025hedging} show that similar non-monosemantic solutions can also be favored to reduce distortion, not only rate (feature hedging). A parallel line of work aims to fix these failures through architectural or training-procedure modifications \citep{gao2024scaling, bussmann2024batchtopksparseautoencoders,bussmann2025learningmultilevelfeaturesmatryoshka,rajamanoharan2024gated,rajamanoharan2024jump}, but such interventions treat polysemanticity as an optimization byproduct to be engineered away, not as a more fundamental data problem. Our work is instead the first to give \textit{a systematic account of the joint relation between rate, distortion, and polysemanticity by treating polysemanticity as a measurable property} of the SAE hypothesis space. This is a genuinely distinct axis: the closest existing measurements, such as MCC in \citet{oneill2025computeoptimalinferenceprovable} and \citet{chanin2026synthsaebench}, evaluate the overall quality of the recovered dictionary rather than polysemanticity itself, an SAE could be perfectly monosemantic yet highly redundant, with many atoms aligned to the same concept direction, and still receive a poor MCC score. In isolating polysemanticity from dictionary quality, we also take a different stance from prior work: \textit{we characterize polysemanticity as a data problem, not only an optimization one.}

%%%%%%%%%%%%%%%%%%%%%%%%%%%%%%%%%%%%%

% \section{Full proxy audit}\label{app:audit}

% This appendix reports the full proxy audit underlying \cref{fig:proxy_audit_bars}: the numeric summary of all ten proxies in \cref{tab:proxy_audit}, and the rate--distortion scatter plots colored by each proxy in \cref{fig:gemma_rd_by_p,fig:relu_gemma_rd_by_p,fig:gpt_specificity}. The necessary condition of \cref{thm:proxy-monotonicity} predicts that, as either the rate or the distortion budget relaxes, the empirical envelope of each proxy should weakly improve; deviations from this pattern are violations of the necessary condition.

\begin{table}[!htbp]
    \centering
    \small
    \begin{tabular}{@{}l cc cc@{}}
        \toprule
        \multirow{2}{*}{\textbf{Proxy} $\hat P$}
           & \multicolumn{2}{c}{\texttt{TopK Gemma}}
           & \multicolumn{2}{c}{\texttt{ReLU Gemma}} \\
        \cmidrule(lr){2-3} \cmidrule(lr){4-5}
           & $V(\hat P)\,\downarrow$ & $\rho(\hat P)\,\uparrow$
           & $V(\hat P)\,\downarrow$ & $\rho(\hat P)\,\uparrow$ \\
        \midrule
        AutoInterp      & 0.886 & $-0.68$ & 0.993 & $-0.63$ \\
        Unlearning      & 0.762 & $-0.54$ & 0.788 & $-0.43$ \\
        Disentanglement & 0.733 & $-0.00$ & 0.857 & $-0.49$ \\
        SCR             & 0.614 & $-0.15$ & 0.855 & $-0.40$ \\
        TPP             & 0.525 & $+0.13$ & 0.862 & $-0.49$ \\
        Isolation       & 0.520 & $+0.11$ & 0.778 & $-0.39$ \\
        Absorption      & 0.480 & $+0.18$ & 0.766 & $-0.23$ \\
        Splitting       & {0.441} & $\mathbf{+0.21}$ & \textbf{0.411} & $\textbf{-0.11}$ \\
        Sparse Probing  & \textbf{0.351} & ${+0.19}$ & 0.709 & $-0.30$ \\
        \bottomrule
    \end{tabular}
    \caption{Audit of SAEBench polysemanticity proxies \citep{karvonen2025saebench} against the RDP monotonicity criterion. $V(\hat P)$ is the envelope-violation rate (lower is better; $V=\tfrac12$ is the random baseline, $V>\tfrac12$ is anti-correlated with the RDP-predicted direction); $\rho(\hat P)$ is the sign-flipped Spearman rank correlation between the joint rate--distortion budget rank and the proxy value (higher is better). Rows are sorted by $V$ on TopK Gemma (worst at top). Only Splitting and Sparse Probing satisfy $V<\tfrac12$ and $\rho>0$ on TopK Gemma.}
    \label{tab:proxy_audit}
\end{table}

\begin{figure}[!htbp]
    \centering
    \includegraphics[width=\linewidth]{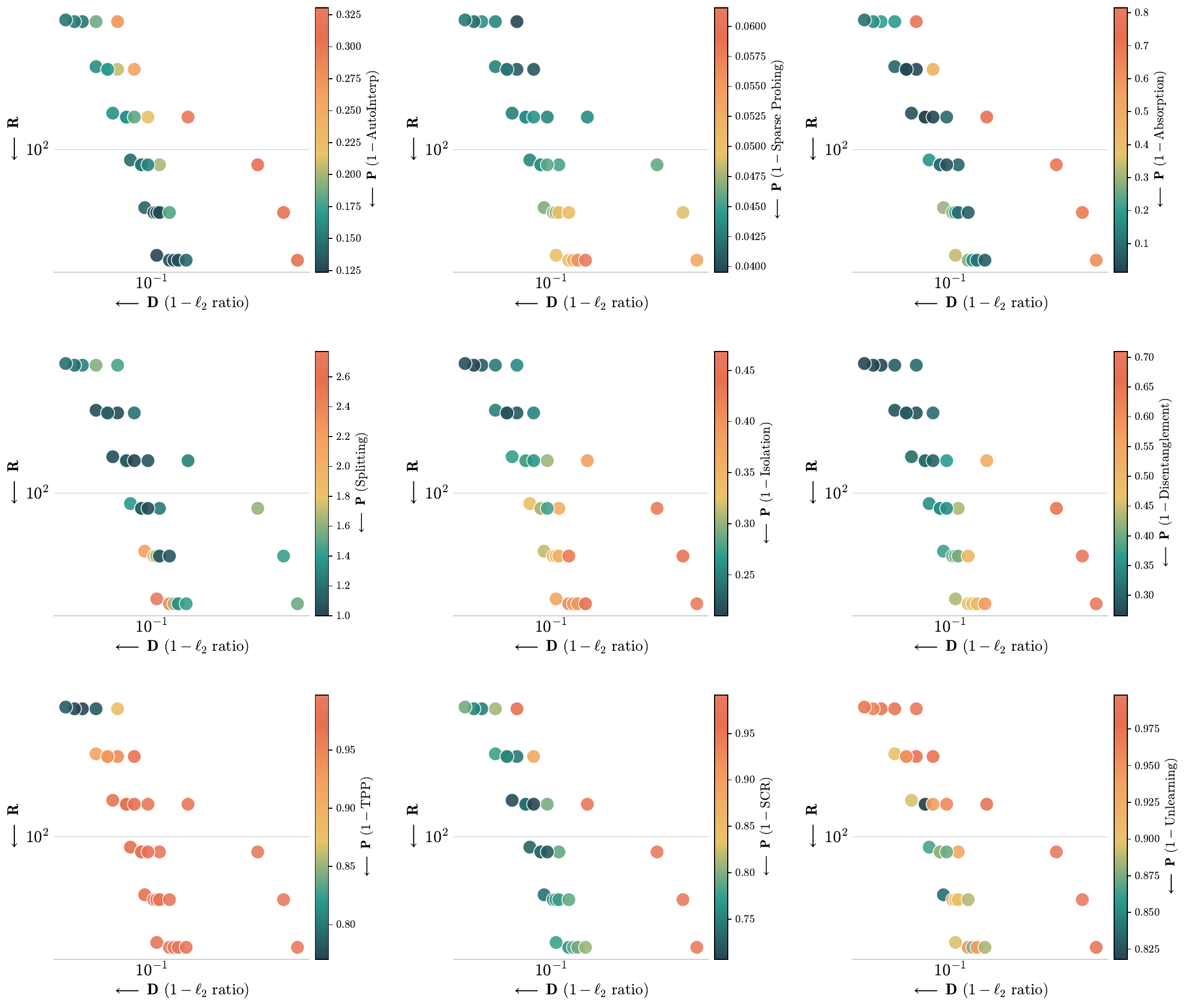}
    \caption{Rate--distortion scatter for TopK SAEs trained on Gemma activations, with each panel colored by one of the nine SAEBench polysemanticity proxies of \citet{karvonen2025saebench}. }
    \label{fig:gemma_rd_by_p}
\end{figure}

\begin{figure}[!htbp]
    \centering
    \includegraphics[width=\linewidth]{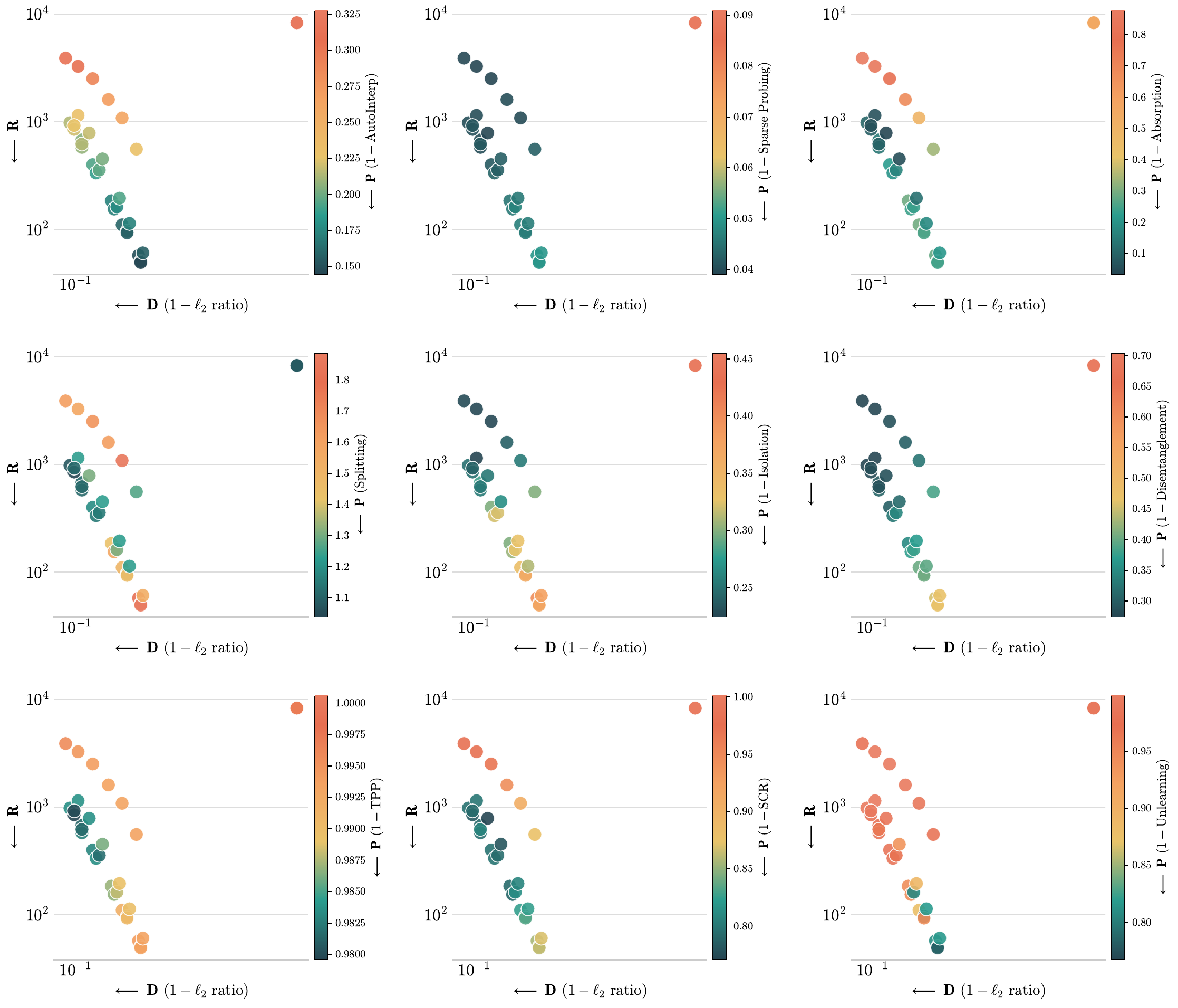}
    \caption{Rate--distortion scatter for ReLU SAEs trained on Gemma activations, colored by each of the nine SAEBench polysemanticity proxies.}
    \label{fig:relu_gemma_rd_by_p}
\end{figure}

\section{Computational resources}\label{app:compute}
The synthetic experiments (Figure 4) were run on a single NVIDIA L40S GPU (48 GB GDDR6, Ada Lovelace). The real-world analysis (Section 4) applies deterministic statistics to pre-computed SAEBench outputs and requires no training; its compute cost is negligible.

\section{A toy model of Rate-Distortion-Polysemanticity tradeoff phenomena: formal treatment}\label{sec:toy}

\subsection{Setting}\label{sec:setting}
As we aim to study the RDP tradeoff from a theoretical perspective, we state some simplifying assumptions on thedata-generatingg process and the SAE architecture that make the problem approachable in closed analytical form. In the text, we specify when each of these assumptions is used, rather then imposing them in general.

\begin{restatable}[Orthogonal true factors]{assumption}{AssOrthogonal}\label{ass:orthogonal_features}
    We assume
\[
\langle v_i,v_\ell\rangle=\delta_{i\ell}.
\]    
\end{restatable}

% Under this orthogonality assumption, the map from $\mathbf x \mapsto S$ is injective, and so observing $\mathbf x = \sum_{\ell =1}^n c_\ell \mathbf v_\ell =  \sum_{\ell \in S} \mathbf v_\ell$ is equivalent to observe the density $p_S$ of the random variable $S = \set{\ell |  c_\ell=1, \ell=1,...,n}$ of co-occurring concepts. Hence, from now on we will reason in terms of $p_S$ rather than $p_{\mathbf x}$, to explain how the parametrization of an SAE captures information about co-occurrence of concepts in $\mathbf x$ samples. 
Under this orthogonality assumption, the map \(S \mapsto \mathbf x(S)=\sum_{\ell\in S} v_\ell\) is injective. Indeed,
\[
\mathbf{1}\{\ell\in S\}=\langle \x(S), v_\ell\rangle.
\]
Hence observing \(\x\) is equivalent to observing the source event \(S\), and the law $p_\x$ may be studied directly through the pmf \(p_S\) of co-occurring concepts: from now on we will reason in terms of $p_S$ to explain how the parametrization of an SAE is driven by the co-occurrence of concepts in $\mathbf x$ samples. 

\begin{restatable}[Zero bias terms]{assumption}{AssZeroBias}\label{ass:zero_bias} We assume $\mathbf b_{\mathrm{enc}} = \mathbf b_{\mathrm{dec}} = 0$.
\end{restatable}

\begin{restatable}[Tied monosemantic SAE]{assumption}{AssTiedness}\label{ass:tiedness}
For monosemantic SAEs, we assume \textit{tied} encoders and decoders, i.e.
\begin{equation}\label{eq:tiedness}
    W_{\textnormal{enc}} = W_{\textnormal{dec}}^T.
\end{equation}
\end{restatable}
Under our definition of polysemanticity of \cref{eq:polysemanticity}, this assumption is not restrictive: in \cref{app:tiedness_wlog} we show that monosemanticity is in fact equivalent to tiedness in the sense of \cref{eq:tiedness}.
This is unrealistic in real networks: concepts are not binary, features are noisy, and directions
need not be orthogonal. We keep the assumption because it isolates the role of the joint probability mass function from geometric complications.

\begin{restatable}[Aligned and binary atoms]{assumption}{AssAlignedAtoms}\label{ass:aligned_atoms}
     The encoder $W_{\textnormal{enc}}$ and decoder $W_\textnormal{dec}$ rows are linear combinations of true concepts $\mathbf v_\ell$  with unit weights. So, for each observation 
     $$
        \mathbf x = \sum_{\ell \in S} \mathbf v_\ell,
     $$
     where $S$ is the index of the groundtruth active concepts, the SAE defines a map $S \mapsto \hat S_\theta \subset \set{1,...,n}$ and the reconstructed observation is 
     \begin{equation}\label{eq:aligned_reconstruction}
         \hat{\mathbf x}_{\theta} = \sum_{\ell \in \hat S_{\theta}} \mathbf v_\ell.
     \end{equation}
\end{restatable}
This is more restrictive than a general SAE. However, note that without loss of generality we can already assume that each reconstruction has the form:
\begin{equation}
    \hat{\mathbf x}_{\theta} = \sum_{\ell \in \hat S_\theta} \hat \alpha_{\theta,\ell} \mathbf v_\ell, \quad \hat \alpha_{\theta,\ell} \in \R.
\end{equation}

Indeed, under \cref{ass:orthogonal_features}, every reconstruction admits a unique orthogonal decomposition
\[
\hat \x_\theta=\sum_{\ell=1}^n \hat\alpha_{\theta,\ell} \mathbf v_\ell + \mathbf r_\perp,
\qquad
\mathbf r_\perp \perp \mathrm{span}\{\mathbf v_1,\dots,\mathbf v_n\}.
\]
Since each ground-truth sample satisfies \(\x \in \mathrm{span}\{\mathbf v_1,\dots,\mathbf v_n\}\), we have
\[
\|\x-\hat \x_\theta\|_2^2
=
\left\|
\x-\sum_{\ell=1}^n \hat\alpha_{\theta,\ell} \mathbf v_\ell
\right\|_2^2
+
\|\mathbf r_\perp\|_2^2.
\]
Hence the orthogonal residual \(\mathbf r_\perp\) can only increase distortion. Therefore, for
distortion-minimization, one may restrict without loss of generality to reconstructions lying in
\(\mathrm{span}\{\mathbf v_1,\dots,\mathbf v_n\}\). \cref{ass:aligned_atoms} only imposes the further simplifying restriction
\(\hat\alpha_{\theta,\ell}\in\{0,1\}\).

\paragraph{Summary.} Assumptions \cref{ass:orthogonal_features} and \cref{ass:zero_bias} are the main simplifying assumptions. \cref{ass:tiedness} is without loss of generality, and made explicit only for clarity.  \cref{ass:aligned_atoms} is an additional tractability assumption that simplifies the theoretical analysis; the main qualitative conclusions extend beyond it.

\subsection{A toy model of the rate-distortion-polysemanticity tradeoff}\label{sec:formal_theorems}
The goal of this section is to theoretically and empirically demonstrate that \textit{(i)} there is a tension between efficiency, reconstruction accuracy, and interpretability, meaning that they can not be jointly optimized. We call this the rate-distortion-polysemanticity tradeoff; \textit{(ii)} the RDP tradeoff is not an optimization or architectural artifact: polysemanticity emerges to exploit (reducing distortion and rate) the properties of the DGP, namely co-occurrences of features that can not be compactly captured by monosemantic features.

Firstly, we characterize the general property of the rate-distortion-polysemanticity trade-off. The next result is completely general and does not use Assumptions \ref{ass:orthogonal_features}-\ref{ass:aligned_atoms}.
\begin{theorem}[Basic monotonicity of the RDP frontier, formalization of \cref{thm:monotonicity_informal}]\label{thm:monotonicity}
Fix a data-generating process $p_\mathbf c$ and a class $\Theta$ of SAEs.
Then:
\begin{enumerate}[(i)]
    \item for every fixed $P_0$, the map $D_0 \mapsto R_p^\star(D_0,P_0)$ is nonincreasing;
    \item for every fixed $D_0$, the map $P_0 \mapsto R_p^\star(D_0,P_0)$ is nonincreasing;
    \item for every fixed $R_0$, the map $P_0 \mapsto D_p^\star(R_0,P_0)$ is nonincreasing.
\end{enumerate}
In particular, for every $D_0$ and $R_0$,
\[
R_p^\star(D_0,\infty)\le R_p^\star(D_0,0),
\qquad
D_p^\star(R_0,\infty)\le D_p^\star(R_0,0).
\]
\end{theorem}
\paragraph{Intuition.} The first part, is an instance of the \textit{usual} RD tradeoff common in information theory. The second and third points are novel. Part (ii) shows that, fixed the distortion, better monosemeanticity can be achieved only at the expense of the rate. Complementary to this, part (iii) demonstrates that at a fixed rate $R$, monosemanticity can only be achieved degrading the reconstruction. These results support the qualitative claim that optimizing reconstruction under rate constraints does not lead to interpretable SAE features. 
\begin{callout}
   The RDP tradeoff implies that, by necessity, better monosemanticity must enforce worse reconstruction or efficiency of the code.
\end{callout}

\begin{proof}
For fixed $P_0$, if $D_1\le D_2$, then
\[
\{\theta\in\Theta : D(\theta)\le D_1,\; P(\theta)\le P_0\}
\subseteq
\{\theta\in\Theta : D(\theta)\le D_2,\; P(\theta)\le P_0\}.
\]
Taking the infimum of $R(\theta)$ over the larger feasible set can only decrease the value, so
\[
R_p^\star(D_2,P_0)\le R_p^\star(D_1,P_0).
\]
This proves (i).

Similarly, for fixed $D_0$, if $P_1\le P_2$, then
\[
\{\theta\in\Theta : D(\theta)\le D_0,\; P(\theta)\le P_1\}
\subseteq
\{\theta\in\Theta : D(\theta)\le D_0,\; P(\theta)\le P_2\},
\]
hence
\[
R_p^\star(D_0,P_2)\le R_p^\star(D_0,P_1).
\]
This proves (ii).

Finally, for fixed $R_0$, if $P_1\le P_2$, then
\[
\{\theta\in\Theta : R(\theta)\le R_0,\; P(\theta)\le P_1\}
\subseteq
\{\theta\in\Theta : R(\theta)\le R_0,\; P(\theta)\le P_2\},
\]
so
\[
D_p^\star(R_0,P_2)\le D_p^\star(R_0,P_1).
\]
This proves (iii). The final inequalities are the special case $0\le \infty$.
\end{proof}

\cref{thm:monotonicity} is tells us that we can not have efficiency together with good reconstruction and interpretability. The statement, however, is only qualitative. If we allow ourselves more strict modeling assumptions, then we can quantify how \textit{worse} (less efficient) our rate should be to enforce monosemanticity together with good reconstruction. 

\begin{theorem}[Rate tax for monosemanticity, formalization of \cref{thm:rate_tax_informal}]
\label{thm:rate_tax}
Fix a DGP with concepts $\mathbf c\sim p_{\mathbf c}$ satisfying
\cref{ass:orthogonal_features}, and a class $\Theta$ of width-$m$ SAEs satisfying
\cref{ass:aligned_atoms}. Let $D_\infty := D_p^\star(R\le k, P=\infty)$ the optimal distortion achieved by an unrestricted (potentially) polysemantic SAE, and define
\[
\Theta_M(D_\infty)
:=
\{\theta\in\Theta:\theta \text{ monosemantic and } D(\theta)\le D_\infty\}.
\]
Assume that:
\begin{enumerate}[(i)]
    \item $\Theta_M(D_\infty)\neq\emptyset$, with $\delta := \max_{\theta\in\Theta_M(D_\infty)} D(\theta)$.
    \item for every monosemantic SAE $\theta$, $R(\theta)\le k \;\Longrightarrow\; D(\theta)>D_\infty.$
\end{enumerate}
Then every monosemantic $\theta\in\Theta_M(D_\infty)$ satisfies
\[
R(\theta)\ge \mathbb E_{p_{\mathbf c}}\|\mathbf c\|_0-\delta > k,
\]
with the first inequality strict if and only if $D(\theta) < \delta$.
\end{theorem}

\paragraph{Intuition.} We consider the optimal SAE of size $m$ for reconstructing data from a DGP, given a rate budget $k$. If the optimal SAE can only be polysemantic, the theorem says the following:  
% Callout with box or something
\begin{callout}
    If there is a monosemantic SAE that can achieve equal or better reconstruction on the same DGP, by necessity, it must be strictly less efficient, i.e. it must increase its rate. Moreover, given $p_{\mathbf c}$, we can quantify a lower bound of how bigger the rate needs to be.
\end{callout}
\begin{proof}
By \cref{prop:monosemantic_distortion}, each $\theta \in \Theta_M$ satisfies:
\begin{equation}\label{eq:d_plus_r_is_s}
D(\theta)+R(\theta)=\mathbb E_{p_{\mathbf c}} ||\mathbf c||_0.
\end{equation}
Hence, under monosemanticity constraints, minimizing the rate is equivalent to maximizing
the attainable distortion subject to the strict cap \(D(\theta)\leq D_\infty\). Under \cref{ass:aligned_atoms}, we can show that $\Theta_M$ is a finite set. Moreover, given assumption \textit{(i)}, it is non-empty. Then, we can take the maximum over the set, such that $\delta = \max_{\theta \in \Theta_M(D\infty)} D(\theta)$ is well defined. Then, from \cref{eq:d_plus_r_is_s}, $\delta$ minimizes the rate attainable by a monosemantic solution and that is still better than the polysemantic optimum. It follows that:
\begin{equation*}
    R(\theta) \geq \mathbb E_{p_{\mathbf c}} ||\mathbf c||_0 - \delta, \quad \theta \in \Theta_M.
\end{equation*}

It remains to prove that this lower bound is strictly larger than \(k\). Suppose instead that $\mathbb E ||\mathbf c||_0 - \delta\le k.$ Then, there would exist a perfectly monosemantic
SAE $\theta \in \Theta_M$ with
\[
D(\theta)\leq D_\infty
\qquad\textnormal{and}\qquad
R(\theta)\le k.
\]
This contradicts the assumption \emph{(i)}. Therefore $\mathbb E_{p_{\mathbf c}} ||\mathbf c||_0 - \delta>k.$
\end{proof}

\subsection{Co-occurrence of events drives polysemanticity}\label{sec:coccurrence}
In the previous section, we formalize the existence of the RDP tradeoff from a qualitative (\cref{thm:monotonicity}) and quantitative (\cref{thm:rate_tax}) point of view. However, it is still unclear why monosemanticity necessarily degrades efficiency or reconstruction. Previous works analyse hedging and feature splitting phenomena, and provide empirical evidence that they are driven by positive correlation of events (i.e., concepts occurring together in a probabilistically dependent way).  In this section we show that the assumptions of a data generating process behind the data allows to theoretically support these empirical claims. In particular, given \cref{lem:loss} we can write the loss function of a TopK-SAE minimizing mean-squared error in closed form as a function of the probability of random events in the data generating process.  This is exploited in the next example, which theoretically characterizes  how polysemanticity emerges to minimize the loss (reconstruction error) under sparsity constraints.

\begin{example}\label{example:coccurrence}
    Consider a data generating process $\mathbf x = \sum_{\ell=1}^3 c_\ell \mathbf v_\ell$, with parameters
    $$
    p_S = \mathbb P(\set{c_\ell=1 : \ell \in S} \textnormal{ and } \set{c_\ell=0 : \ell \not \in S}), \:\: S \subseteq \{1, 2, 3\}.
    $$
    Hence, we consider with three orthogonal ground-truth concepts, and each event $S$ describes their co-occurrence in $\mathbf x(S)$. To be in the narrow regime, we consider SAE with a dictionary size $m=2$. To upperbound the rate, we assume TopK activation with $K \in \set{1, 2}$. In particular
    \begin{enumerate}[(i)]
        \item A monosemantic SAE $\theta_M^{ij}$ consists of two monosemantic atoms which write $\mathbf v_i$ when $c_i=1$, and $\mathbf v_j$ when $c_j=1$.
        \item An \textit{hedged} SAE $\theta_H^{ij,k}$ consists of two atoms, one writing $\mathbf v_i+\mathbf v_j$ when $c_i=1 \land c_j =1$, and a monosemantic atom writing $\mathbf v_k$ when $c_k=1$.
        \item A \textit{feature splitting} SAE $\theta_S^{ij,i}$ consists of a first atom writing $\mathbf v_i+\mathbf v_j$ when $c_i=1 \land c_j=1$ and a second atom writing $\mathbf v_i$ when $c_i=1 \land c_j=0$.
    \end{enumerate}
    Let $\mathcal L(\theta)$ be the mean squared reconstruction error achieved  by an SAE $\theta$. For a polysemantic solution $\theta_P$ to be better than a monosemantic $\theta_M$, we need to find the set of DGP parameters $\mathbf c \sim p_{\mathbf c}$ such that
    $$
    \mathcal L(\theta_M) - \mathcal L(\theta_P) > 0. 
    $$
    Consider Assumptions \ref{ass:orthogonal_features}-\ref{ass:aligned_atoms} satisfied. Then, by \cref{lem:loss}, we can compute the mean squared reconstruction error in closed form, which allows verifying when hedging and splitting phenomena are convenient with respect to monosemantic SAEs.
    % 
    % Need to indent a bit the K=2.
    \paragraph{K=2.} Consider the case where the best monosemantic SAE encodes features $\mathbf v_j, \mathbf v_k$, $j \neq k$. Then, we can provide necessary and sufficient conditions for hedging and splitting to occur, comparing the distortion achieved by \textit{hedged} and \textit{feature-splitted SAEs}. Given $K=2$ in TopK activations, the following hold:
    \begin{align*}
        &\mathcal L(\theta_M^{jk}) - \mathcal L(\theta_H^{ij, k}) > 0 \iff p_{ij} + p_{ijk} > p_j + p_{jk}.\\
        & L(\theta_M^{jk}) - L(\theta_H^{jk,i}) > 0 \iff p_i + p_{ijk} > p_j + p_k \\
        & \mathcal L(\theta_M^{jk}) - \mathcal L(\theta_{S}^{ij, i}) >0 \iff  p_i + p_{ij} > p_j + p_k + 2p_{jk}\\
        & \mathcal L(\theta_M^{jk}) - \mathcal L(\theta_{S}^{ij, j}) >0 \iff  p_{ij} > p_k + p_{ik} + p_{jk}
    \end{align*}
    Clearly, given that indices are silent, all the above similarly holds if we swap $j$ and $k$. The remaining hedged and splitted solutions that we do not explicitly compare are never better than the monosemantic ones. 
    \begin{callout}
        A polysemantic atom writing $\mathbf v_i + \mathbf v_j$ (either via hedging or splitting) is convenient when the probabilities $p_{ij}$ and $p_{ijk}$ of co-occurrence of $\mathbf v_i$ and $\mathbf v_j$ in the data generating process are big with respect to the probability of other co-occurring  concepts.
    \end{callout}
    This is the single most imporant outcome of the section, which generalizes to arbitrary size DGP and narrow SAEs: polysemantically merging features is convenient whenever they \textit{sufficiently} co-occur in the groundtruth data generating process.

    \paragraph{K=1.} We now consider $K=1$ upperbounding the rate via TopK activations. Similarly to the above case, we assume that the best monosemantic SAE encodes features $\mathbf v_j, \mathbf v_k$, $j \neq k$. Then, we can characterize the sufficient and necessary conditions for hedging or splitting to occur, specifying when $\mathcal L (\theta_M^{ij}) - \mathcal L(\theta)>0$ for a polysemantic solution $\theta$.
    \begin{align*}
    &\mathcal L(\theta_M^{jk}) - \mathcal L(\theta_H^{ij,k}) > 0 \iff p_{ij} + p_{ijk} > p_j \\
    &\mathcal L(\theta_M^{jk}) - \mathcal L(\theta_H^{jk,i}) > 0 \iff p_i + p_{jk} + p_{ijk} > p_j + p_k \\
    &\mathcal L(\theta_M^{jk}) - \mathcal L(\theta_S^{ij,i}) > 0 \iff p_i + p_{ij} + p_{ijk} > p_j + p_k + p_{jk}\\
    &\mathcal L(\theta_M^{jk}) - \mathcal L(\theta_S^{ij,j}) > 0 \iff p_{ij} + p_{ijk} > p_k + p_{ik} \\
    &\mathcal L(\theta_M^{jk}) - \mathcal L(\theta_S^{jk,j}) > 0 \iff p_{jk} + p_{ijk} > p_{k} + p_{ik} \\
    \end{align*}
As for $K=2$, merging $\mathbf v_i$ and $\mathbf v_j$ in one atom writing $\mathbf v_i + \mathbf v_j$ when both concepts occur is convenient when $p_{ij} + p_{ijk}$ probability of co-occurrence is high enough. A similar reasoning shows that merging $\mathbf v_j, \mathbf v_k$ is convenient when they co-occurr frequently enough. 
\end{example}
% \FM{In the appendix, we show a counterexample that Chanin paper is wrong, namely hedging is not only driven by correlation 
% 1. Show one example where correlation is very high and hedging is not convenient.
% 2. Show an example where correlation is null, and hedging is convenient.
% }

\paragraph{Intuition.} The example above tells that, for the three variable setting we considered, we can provide a closed-form characterization of when polysemanticity arise. In particular:
\begin{callout}
    Polysemanticity predictably arise to exploit co-occurrences between features that can not be captured by monosemantic SAEs, in order to comply to rate constraints (as, e.g., in TopK-SAEs; ReLU-SAEs) while minimizing the distortion.
\end{callout}
This is a key, novel message of our paper. Previous work, only empirically showed the emergence of polysemantic phenomena (splitting and hedging), and explained them via correlation. Instead, we move from an empirical to a theoretical characterization, and show that polysemanticity is the result of co-occurrence of features, rather than correlation: in other words, we only care that the probability $\mathbb P(c_i=1, c_j =1)$ is high, rather than about positive correlation of $c_i, c_j$. The theoretical characterization is the result of a formalization of a data generating process, which is implicit in all toy models in the literature, yet never formalzied and hence exploited. Moreover, previous works were limited in their conclusion: they showed that splitting and hedging where the result exclusively the result of the rate constraints (splitting) or the need of minimizing distortion (hedging), and that splitting phenomena where limited to wide SAE. Our theoretical results suggest something different: 
\begin{itemize}
    \item  Both splitting and hedging can emerge in narrow SAE. 
    \item Both splitting and hedging are the result of the joint need of minimizing the distortion under rate constraints.
\end{itemize}

\subsection{Tiedness of monosemantic SAEs is without loss of generality}\label{app:tiedness_wlog}

The next proposition shows that, in the present orthogonal toy model, every deterministic perfectly monosemantic SAE is weakly dominated by a tied one, so that \cref{ass:tiedness} of $W_{\textrm{enc}} = W_{\textrm{dec}}^\top$ in monosemantic SAEs is without loss of generality.

\begin{proposition}[Tiedness is without loss of generality at \(P(\theta)=0\)]\label{prop:tied_wlog}
Let \cref{ass:orthogonal_features} (orthogonality) and \cref{ass:aligned_atoms} (SAE atoms are binary combinations of groundtruth concepts) satisfied. Let \(\theta\) be a deterministic width-\(m\) SAE with \(P(\theta)=0\). Then there exists a tied deterministic width-\(m\) SAE \(\tilde\theta\) such that
\[
P(\tilde\theta)=0,\qquad
R(\tilde\theta)\le R(\theta),\qquad
D(\tilde\theta)\le D(\theta).
\]
Consequently, the infimum in \(R^\star(D_0,0)\) may be restricted, without loss of optimality, to tied monosemantic SAEs.
\end{proposition}

\begin{proof}
Since $P(\theta)=0$ and \cref{ass:aligned_atoms} holds, each nonzero latent of $\theta$ reconstructs exactly one ground-truth concept. Therefore, for each latent $r\in\{1,\dots,m\}$, there is a unique index $\ell(r)\in\{1,\dots,n\}$ such that latent $r$, when active, contributes the vector $\bfv_{\ell(r)}$ to the reconstruction. Writing the latent activity on input $\x(S)$ as $a_r(S)\in\{0,1\}$, we may write
$$
\hat \x_\theta(S)=\sum_{r=1}^m a_r(S)\,\bfv_{\ell(r)}.
$$
By \cref{ass:orthogonal_features}, $\x$ uniquely determines $S$, so this notation is well defined.

Now define a modified SAE $\theta'$ by suppressing every latent activation that reconstructs a concept not present in the input:
$$
a_r'(S):=a_r(S)\mathbf 1\{\ell(r)\in S\}.
$$
Then
$$
\hat \x_{\theta'}(S)=\sum_{r=1}^m a_r'(S)\,\bfv_{\ell(r)}.
$$
Since $a_r'(S)\leq a_r(S)$ for every $r$ and $S$, we have
$$
R(\theta')\leq R(\theta).
$$
Moreover, if $\ell(r)\notin S$, then the term $a_r(S)\bfv_{\ell(r)}$ adds to the reconstruction a direction orthogonal to $\x(S)$. Removing such a term cannot increase the squared reconstruction error. Hence
$$
D(\theta')\leq D(\theta).
$$

Let the set $I$ collects the ground-truth concepts that are reconstructed by at least one latent of $\theta'$. Formally,
$$
I:=\{\ell(r):1\leq r\leq m\}.
$$
For each $\ell\in I$, let $b_\ell(S)\in\{0,1\}$ indicate whether, on input $\x(S)$, at least one latent of $\theta'$ assigned to concept $\ell$ is active:
$$
b_\ell(S):=\mathbf 1\{\exists r \textnormal{ such that } \ell(r)=\ell \textnormal{ and } a_r'(S)=1\}.
$$
Now define $\tilde \theta$ by assigning one latent to each $\ell\in I$, with reconstruction
$$
\hat \x_{\tilde \theta}(S)=\sum_{\ell\in I} b_\ell(S)\,\bfv_\ell,
$$
and, if $|I|<m$, adding $m-|I|$ identically zero latents. By construction, $\tilde \theta$ is tied and $P(\tilde \theta)=0$.

To compare $\tilde \theta$ and $\theta'$, fix $S$ and $\ell\in I$, and let $m_\ell(S)$ denote how many active latents of $\theta'$ reconstruct concept $\ell$ on input $\x(S)$:
$$
m_\ell(S):=\sum_{r:\ell(r)=\ell} a_r'(S).
$$
Then the coefficient of $\bfv_\ell$ in $\hat \x_{\theta'}(S)$ is $m_\ell(S)$, whereas the coefficient of $\bfv_\ell$ in $\hat \x_{\tilde \theta}(S)$ is $b_\ell(S)=\mathbf 1\{m_\ell(S)\geq 1\}$. Since $a_r'(S)=0$ whenever $\ell(r)\notin S$, we have $m_\ell(S)=0$ if $c_\ell=0$. Hence
$$
\big(c_\ell-b_\ell(S)\big)^2 \leq \big(c_\ell-m_\ell(S)\big)^2.
$$
Summing over $\ell$ and using \cref{ass:orthogonal_features} gives
$$
D(\tilde \theta)\leq D(\theta').
$$
Moreover, for every $S$,
$$
\sum_{\ell\in I} b_\ell(S)\leq \sum_{r=1}^m a_r'(S),
$$
so
$$
R(\tilde \theta)\leq R(\theta').
$$
Since already $D(\theta')\leq D(\theta)$ and $R(\theta')\leq R(\theta)$, we conclude
$$
D(\tilde \theta)\leq D(\theta), \qquad R(\tilde \theta)\leq R(\theta).
$$
This proves the claim.
\end{proof}

\subsection{Exact geometry of the perfectly monosemantic frontier}\label{app:monosemantic_frontier}

We now characterize $R^\star(D_0,0)$ exactly in the orthogonal toy model. By \cref{prop:tied_wlog}, it suffices to consider tied monosemantic SAEs. For such codes, the SAE may reconstruct only concepts in a fixed set $I\subseteq\{1,\dots,n\}$, and for each sample with active set $S$ it outputs some $\hat S_\theta\subseteq S\cap I$.

\begin{proposition}[Exact rate-distortion identity for a fixed represented set]\label{prop:monosemantic_distortion}
Assume \ref{ass:orthogonal_features} and \ref{ass:aligned_atoms}. Fix $I\subseteq\{1,\dots,n\}$ with $|I|=m$, and let $\theta$ be a tied monosemantic width-$m$ SAE such that $\hat S_\theta\subseteq S\cap I$ for every sample $\x(S)$. Then
$$
D(\theta)+R(\theta)=\mathbb E_{p_{\bfc}}\|\bfc\|_0.
$$
Moreover,
$$
D(\theta)\geq \sum_{\ell\notin I}\mathbb P(c_\ell=1),
$$
with equality if and only if $\hat S_\theta=S\cap I$ almost surely.
\end{proposition}

\begin{proof}
For a sample with active set $S$,
$$
x=\sum_{\ell\in S} \bfv_\ell,
\qquad
\hat \x_\theta=\sum_{\ell\in \hat S_\theta} \bfv_\ell,
$$
with $\hat S_\theta\subseteq S\cap I$. Since \cref{ass:orthogonal_features} gives orthonormal concept directions,
$$
\|x-\hat \x_\theta\|_2^2
=
\sum_{\ell\notin I} c_\ell
+
\sum_{\ell\in I}\big(c_\ell-\mathbf 1\{\ell\in \hat S_\theta\}\big)^2.
$$
Because $\mathbf 1\{\ell\in \hat S_\theta\}\leq c_\ell$ and both terms are binary,
$$
\big(c_\ell-\mathbf 1\{\ell\in \hat S_\theta\}\big)^2
=
c_\ell-\mathbf 1\{\ell\in \hat S_\theta\}.
$$
Hence
$$
\|x-\hat \x_\theta\|_2^2
=
\sum_{\ell=1}^n c_\ell-\sum_{\ell=1}^n \mathbf 1\{\ell\in \hat S_\theta\}
=
\|c\|_0-|\hat S_\theta|.
$$
Taking expectation and using $R(\theta)=\mathbb E_{p_\bfc}|\hat S_\theta|$ gives
$$
D(\theta)=\mathbb E_{p_\bfc}\|\bfc\|_0-R(\theta).
$$
The lower bound follows because every omitted concept $\ell\notin I$ contributes one unit of distortion whenever $c_\ell=1$. Equality holds if and only if the SAE reconstructs every active represented concept, namely $\hat S_\theta=S\cap I$ almost surely.
\end{proof}

\begin{theorem}[Exact $P(\theta)=0$ frontier]\label{thm:exact_monosemantic_frontier}
Assume \ref{ass:orthogonal_features} and \ref{ass:aligned_atoms}. Then
$$
R^\star(D_0,0)=
\begin{cases}
+\infty, \quad
D_0<
\displaystyle
\min_{\substack{I\subseteq\{1,\dots,n\}\\ |I|=m}}
\sum_{\ell\notin I}\mathbb P(c_\ell=1),
\\[3mm]
\mathbb E_{p_\bfc}\|\bfc\|_0-
\displaystyle
\max\{D(\theta): \theta\in\Theta,\ P(\theta)=0,\ D(\theta)\leq D_0\},
\quad
\textnormal{otherwise}.
\end{cases}
$$
In particular, the perfectly monosemantic frontier is a staircase rather than a line.
\end{theorem}

\begin{proof}
By \cref{prop:tied_wlog}, the infimum defining $R^\star(D_0,0)$ may be restricted to tied monosemantic SAEs, and by the clipping step in the proof of \cref{prop:tied_wlog} we may further restrict to omission codes satisfying $\hat S_\theta\subseteq S$. For every such code, \cref{prop:monosemantic_distortion} gives
$$
R(\theta)=\mathbb E_{p_\bfc}\|\bfc\|_0-D(\theta).
$$
Therefore minimizing the rate under the constraint $D(\theta)\leq D_0$ is equivalent to maximizing the attainable monosemantic distortion under the same cap, which yields the second line.

For the first line, if $\theta$ represents the set $I$, \cref{prop:monosemantic_distortion} implies
$$
D(\theta)\geq \sum_{\ell\notin I}\mathbb P(c_\ell=1).
$$
Hence no perfectly monosemantic width-$m$ SAE is feasible below
$$
\min_{\substack{I\subseteq\{1,\dots,n\}\\ |I|=m}}
\sum_{\ell\notin I}\mathbb P(c_\ell=1).
$$

Finally, under \cref{ass:aligned_atoms}, the set of attainable monosemantic distortions is finite, so the maximizer above changes only at finitely many values of $D_0$. Thus $D_0\mapsto R^\star(D_0,0)$ is piecewise constant, namely a staircase.
\end{proof}

%%%%%%%%%%%%% Checklist %%%%%%%%%%%%%%%%%5
% \newpage
% \input{checklist}
\end{document}